\documentclass{article}

\usepackage{microtype}
\usepackage{graphicx}
\usepackage{subfigure}
\usepackage{booktabs} 

\usepackage{hyperref}

\usepackage[accepted]{icml2024}

\usepackage{amsmath}
\usepackage{amssymb}
\usepackage{mathtools}
\usepackage{amsthm}

\usepackage[capitalize,noabbrev]{cleveref}

\theoremstyle{plain}
\newtheorem{theorem}{Theorem}[section]
\newtheorem{proposition}[theorem]{Proposition}

\theoremstyle{definition}
\newtheorem{definition}[theorem]{Definition}

\theoremstyle{remark}

\usepackage[textsize=tiny]{todonotes}

\usepackage{amsthm}
\usepackage{amsmath}
\usepackage{amssymb}
\usepackage{stmaryrd}
\usepackage{booktabs}
\usepackage{bm}
\usepackage{tikz}
\usepackage{graphicx}
\usepackage{color}

\numberwithin{table}{section}
\numberwithin{algorithm}{section}
\numberwithin{equation}{section}

\usepackage{xspace}
\makeatletter
\DeclareRobustCommand\onedot{\futurelet\@let@token\@onedot}
\def\@onedot{\ifx\@let@token.\else.\null\fi\xspace}

\def\eg{\emph{e.g}\onedot}

\def\ie{\emph{i.e}\onedot}

\icmltitlerunning{Tackling Non-Stationarity in Reinforcement Learning via Causal-Origin Representation}

\begin{document}

\twocolumn[
\icmltitle{Tackling Non-Stationarity in Reinforcement Learning via\\ Causal-Origin Representation}



\icmlsetsymbol{equal}{*}

\begin{icmlauthorlist}
\icmlauthor{Wanpeng Zhang}{pku_cs}
\icmlauthor{Yilin Li}{pku_stat}
\icmlauthor{Boyu Yang}{fdu}
\icmlauthor{Zongqing Lu}{pku_cs,baai}
\end{icmlauthorlist}

\icmlaffiliation{pku_cs}{School of Computer Science, Peking University}
\icmlaffiliation{pku_stat}{Center for Statistical Science, Peking University}
\icmlaffiliation{fdu}{School of Data Science, Fudan University}
\icmlaffiliation{baai}{Beijing Academy of Artificial Intelligence}

\icmlcorrespondingauthor{Zongqing Lu}{zongqing.lu@pku.edu.cn}

\icmlkeywords{Reinforcement Learning, Non-Stationarity, Causal Structure}

\vskip 0.3in
]



\printAffiliationsAndNotice{}  

\begin{abstract}
In real-world scenarios, the application of reinforcement learning is significantly challenged by complex non-stationarity. Most existing methods attempt to model changes in the environment explicitly, often requiring impractical prior knowledge of environments. In this paper, we propose a new perspective, positing that non-stationarity can propagate and accumulate through complex causal relationships during state transitions, thereby compounding its sophistication and affecting policy learning. We believe that this challenge can be more effectively addressed by implicitly tracing the causal origin of non-stationarity. To this end, we introduce the \textbf{C}ausal-\textbf{O}rigin \textbf{REP}resentation (\textbf{COREP}) algorithm. COREP primarily employs a guided updating mechanism to learn a stable graph representation for the state, termed as causal-origin representation. By leveraging this representation, the learned policy exhibits impressive resilience to non-stationarity. We supplement our approach with a theoretical analysis grounded in the causal interpretation for non-stationary reinforcement learning, advocating for the validity of the causal-origin representation. Experimental results further demonstrate the superior performance of COREP over existing methods in tackling non-stationarity problems. The code is available at \href{https://github.com/PKU-RL/COREP}{https://github.com/PKU-RL/COREP}.
\end{abstract}

\section{Introduction}

Rapid advancements in reinforcement learning (RL) \citep{kaelbling1996reinforcement, sutton2018reinforcement} have led to significant performance gains in various domains \citep{silver2018general, mirhoseini2021graph}. However, a common assumption in many RL algorithms is the stationarity of the environment, which can limit the applicability in real-world scenarios characterized by varying dynamics \citep{padakandla2020reinforcement, padakandla2021survey}. Recent efforts in the meta-RL approaches \citep{finn2017model} have attempted to tackle this by enabling algorithms to adapt to changes \citep{poiani2021meta}. However, these methods struggle in the face of more complex and unpredictable environmental dynamics \citep{sodhani2022block, feng2022factored}. Approaches like FN-VAE \citep{feng2022factored} and LILAC \citep{xie2020deep} have made strides towards improving RL algorithms in non-stationary environments by explicitly modeling the change factors of the environment. Nevertheless, they may not comprehensively capture the complexity of real-world non-stationarity. This gap highlights the need for a more robust approach to handle the intricacies of non-stationary environments in RL.

In this paper, we propose a novel setting for efficiently tackling non-stationarity in RL from a new perspective inspired by the causality literature \citep{zhang2020domain, huang2020causal}. We contend that minor changes in dynamics can cause significant shifts in observations due to their propagation through intricate causal relationships in the dynamics. Thus, we need to trace the ``causal origin'' of these changes. However, directly constructing a causal graph that captures such information is a significant challenge due to the inherent instability in non-stationary environments \citep{strobl2019improved}. To address this, we introduce the Causal-Origin Representation (COREP) algorithm. COREP employs a guided update mechanism, which enables the learning of a stable graph representation of state, termed as \textit{causal-origin representation}. This representation aims to capture the underlying causal structure in a way that is resilient to the unpredictable changes characteristic of non-stationary environments, aiding RL algorithms to adapt and learn policies in such settings.

Specifically, we first propose a novel formulation of non-stationarity in RL as the mixture of decomposed sub-environments. In this formulation, we rewrite dynamics functions using masks to represent causal relationships, assumed invariant within each sub-environment. However, identifying these causal relationships without prior knowledge of the sub-environments poses a significant challenge. To overcome this, we propose the environment-shared union graph that captures causal relationship information among state elements. This is achieved by combining Maximal Ancestral Graphs (MAGs) from each sub-environment. We also provide theoretical support for the feasibility of recovering this environment-shared union graph. 

To effectively learn the proposed graph representation, we design a dual graph structure comprising a core-graph and a general-graph. The core graph focuses on learning a graph representation that is stable to dynamic changes, guided by a TD (Temporal Difference) error based updating mechanism. While it concentrates on learning the most essential parts of the graph representation, some edge information might be overlooked. To address this, we employ a continuously updating general-graph to compensate for potential information loss and enhance the algorithm's adaptability. By integrating the core-graph and general-graph, we can finally construct the causal-origin representation, providing a comprehensive understanding of the environment's dynamics and significantly mitigating the impact of non-stationarity problems in RL.

Our main contributions can be summarized as follows: 
\begin{itemize}
    \item We provide a causal interpretation for non-stationary RL and propose a novel setting that focuses on the causal relationships within states;
    \item Based on the proposed formulation and setting, we design a modular algorithm that can be readily integrated into existing RL algorithms;
    \item We provide a theoretical analysis that offers both inspiration and theoretical support for our algorithm. Experimental results further demonstrate the effectiveness of our algorithm.
\end{itemize}

\section{Preliminaries}

\textbf{Problem Formulation.} Reinforcement learning problems are typically modeled as Markov Decision Processes (MDPs), defined as a tuple $(\mathcal{S}, \mathcal{A}, \mathcal{P}, \mathcal{R}, \gamma)$, where $\mathcal{S}$ is the state space, $\mathcal{A}$ is the action space, $\mathcal{P}: \mathcal{S} \times \mathcal{A} \times \mathcal{S} \rightarrow [0, 1]$ represents the transition probability, $\mathcal{R}: \mathcal{S} \times \mathcal{A} \rightarrow \mathbb{R}$ is the reward function, and $\gamma \in [0, 1)$ is the discount factor. We may also use the form of a dynamics function: $\bm{s}^\prime = f(\bm{s},\bm{a}, \varepsilon)$ when the environment is deterministic. Here, $\varepsilon$ represents random noise. The goal of an agent in RL is to find a policy $\pi: \mathcal{S} \rightarrow \mathcal{A}$ that maximizes the expected cumulative discounted reward, defined as the value function $V^{\pi}(\bm{s}) = \mathbb{E}_{\pi}[\sum_{t=0}^{\infty} \gamma^t r_t | \bm{s}_0 = \bm{s}]$, where $r_t$ is the reward at time step $t$.

In non-stationary environments, the dynamics of the environment can change over time. Our goal is to learn a policy $\pi$ that can adapt to the non-stationary environment and still achieve high performance. For simplicity of notations, we provide theoretical analysis in the form of the deterministic dynamics function $f$.

\textbf{Causal Structure Discovery.} Causal structure discovery usually aims at inferring causation from data, modeled with a directed acyclic graph (DAG) $\mathcal{D}=(V,E)$, where the set of nodes $V$ includes the variables of interest, and the set of directed edges $E$ contains direct causal effects between these variables \citep{pearl2000Causality}.
The causal graph is a practical tool that relates the conditional independence relations in the generating distribution to separation statements in the DAG ($d$-separation) through the Markov property \citep{lauritzen1996graphical}. If there exist unobserved confounders in the dynamics, maximal ancestral graphs (MAGs) $\mathcal{M}=(V,D,B)$ are often used to represent observed variables by generalizing DAGs with bidirected edges which depicts the presence of latent confounders \citep{richardson2002ancestral}. The sets $D,B$ stand for directed and bidirected edges, respectively.

We also provide the definition of the \textit{partial order}, a basic concept necessary for the theoretical analysis in the manuscript. A partial order, $\pi$, on a DAG is defined to represent a relationship between nodes where their order is not strictly defined but still satisfies three properties as follows: (1) Reflexivity: each node is related to itself \ie, $A <_{\pi} A$; (2) Antisymmetry: if $A <_{\pi} B$ and $B <_{\pi} A$, then $A =_{\pi} B$; (3) Transitivity: if $A <_{\pi} B$ and $B <_{\pi} C$, then $A <_{\pi} C$. When this relationship between $A$ and $B$ is not defined, \ie, neither $A <_{\pi} B$ nor $B <_{\pi} A$, we refer to this case as $A \not\lessgtr B$ with $\pi$, or $A \not\lessgtr_\pi B$.

\textbf{Graph Neural Networks (GNNs).} GNNs are a class of neural networks designed for graph-structured data \citep{scarselli2008graph, zhou2020graph}. Given a graph $\mathcal{G}=(V,E)$, GNNs aim to learn a vector representation for each node $v\in V$ or the entire graph $\mathcal{G}$, leveraging the information of both graph structure and node features. 
GNNs generally follow the message passing framework~\citep{DBLP:conf/iclr/BalcilarRHGAH21}. Layers in GNN can be formulated as ${\bm{H}}^{(l)}=\sigma\big(\sum_s{\bm{L}}_s{\bm{H}}^{(l-1)}{\bm{W}}_s^{(l)}\big)$, where ${\bm{H}}^{(l)}$ is the node representation outputted by the $l$-th layer, ${\bm{L}}_s$ is the $s$-th convolution support which defines how the node features are propagated, ${\bm{W}}_s^{(l)}$ is learnable parameters for the $s$-th convolution support in the $l$-th layer, and $\sigma(\cdot)$ is the activation function. Graph Attention Network (GAT)~\citep{velivckovic2017graph} is a special type of GNN following the message passing framework. Instead of handcrafting, the self-attention mechanism is used to compute the support convolutions in each GAT layer, where the adjacency matrix plays the role of a mask matrix for computing the attention.  

\section{Methodology}

\subsection{Motivation and Key Idea}\label{sec:idea}

In the COREP algorithm, the primary goal is to tackle non-stationarity problems in RL by learning the underlying graph structure of the environment, which we term \textit{causal-origin representation}. This representation strives to be both causal, meaning it can reflect the cause-effect relationships among state elements, and stable, meaning it is robust to the environment's changes. To achieve this, we design a dual GAT structure consisting of a core-GAT and a general-GAT. The core-GAT focuses on learning the stable part of the environment's causal graph, with its learning guided by a specific updating mechanism. The general-GAT, on the other hand, is continually updated to capture any additional information that the core-GAT might overlook. Together, these two GATs form a comprehensive understanding of the environment.

Specifically, in constructing the causal-origin representation, we first transform the states of the environment into node features and create a weighted adjacency matrix to represent the connections between these features. We then apply a self-attention mechanism, which helps the algorithm focus on the most relevant parts of each node. The update of the core-GAT is guided by a TD error detection mechanism, which helps identify the most significant changes in the environment. To further enhance learning efficiency, we integrate the causal-origin representation into a Variational Autoencoder (VAE) framework \citep{kingma2013auto}. Additionally, we introduce regularization terms to improve the identifiability of causal relationships and to ensure the structural integrity of the causal-origin representation. The overall framework is shown in Figure \ref{fig:COR-framework}.

\subsection{Causal Interpretation of Non-Stationary RL} \label{sec:theory}

In this part, we will propose a causal interpretation of non-stationarity in RL, which provides us with inspiration and theoretical support for the algorithm design.
First, for the standard dynamics function $\bm{s}^\prime = f(\bm{s},\bm{a},\varepsilon)$ in RL, in order to better discuss the relationship between elements, we rewrite it as 
\begin{equation}\label{eq:transition}
\begin{gathered}
s_{i}^\prime = f\left({\bm{c}}_{i}^{s \shortrightarrow s} \odot {\bm{s}}, {\bm{c}}_{i}^{a \shortrightarrow s} \odot {\bm{a}}, \varepsilon\right), \\
\bm{s}^\prime = \left(s_1^\prime,\ldots,s_{d_s}^\prime\right),
\end{gathered}
\end{equation}
where $\bm{s}$ denotes the state with dimension $d_s$, $s^\prime_i$ is the $i$-th element of next state $\bm{s}^\prime$, $\varepsilon$ is random noise, and $\odot$ denotes the Hadamard product (element-wise product). The masks $ {\bm{c}}^{ \cdot \shortrightarrow s} $ represent the structural dependence among elements in the following way, \eg, the $j$-th element of ${\bm{c}}_{i}^{s \shortrightarrow s}\in\{0,1\}^{d_s}$ equals $1$ if and only if $s_{j}$ causally affects $s_{i}$. Similarly, ${\bm{c}}_{i}^{a \shortrightarrow s}$ is also defined in the same way. In particular, if only the $i$-th element in ${\bm{c}_i}^{ \cdot \shortrightarrow s}$ equals $1$ and all others are $0$, it represents that each element only causally affects itself. In this case, Equation (\ref{eq:transition}) simplifies to a standard dynamics function.

As mentioned before, when the environment becomes non-stationary, the causal relationships between elements become more intricate. In addition to the relationships between states and actions, there may also be external factors causing environmental changes. These factors can have causal relationships with both state and action, yet are not explicitly considered in Equation (\ref{eq:transition}). Intuitively, we assume the presence of hidden states $\bm{h}$ in the dynamics. Similarly, we define the underlying dynamics function:
\begin{equation}\label{eq:underlying-transition}
\begin{gathered}
h_{i}^\prime=g\left({\bm{c}}^{h \shortrightarrow h}_{i} \odot {\bm{h}}, {\bm{c}}^{s \shortrightarrow h}_{i} \odot {\bm{s}}, {\bm{c}}^{a \shortrightarrow h}_{i} \odot {\bm{a}}, \varepsilon\right), \\
\bm{h}^\prime = \left(h_1^\prime,\ldots,h_{d_h}^\prime\right).
\end{gathered}
\end{equation}

Considering the influence of the hidden state, the reward function can be similarly rewritten as:
\begin{equation}\label{eq:reward-function}
r=\mathcal{R}\left({\bm{c}}^{s \shortrightarrow r} \odot {\bm{s}}, {\bm{c}}^{h \shortrightarrow r} \odot {\bm{h}}, {\bm{c}}^{a \shortrightarrow r} \odot {\bm{a}}, \varepsilon\right).
\end{equation}

The masks can be combined into matrix form as ${\bm{C}}^{\cdot \shortrightarrow s}:=[{\bm{c}}_{i}^{\cdot \shortrightarrow s}]_{i=1}^{d_s}$, ${\bm{C}}^{\cdot \shortrightarrow h}:=[{\bm{c}}_{i}^{\cdot \shortrightarrow h}]_{i=1}^{d_h}$.
Some previous work assumed that similar masks are invariant over time, and simply encoded them into some change factors \citep{huang2022adarl}. 
Instead, we allow such ${\bm{C}}^{\cdot \shortrightarrow \cdot}$ to be time-varying, and propose a novel causal interpretation based on an environment-shared union graph representation to capture the transition information in the non-stationary environment.

With the defined underlying dynamics, we can represent the non-stationarity as being governed by specific causal relationships. Specifically, we regard a non-stationary environment as a mixture of various stationary sub-environments. The non-stationarity is thus interpreted as changes over time in the mixture distribution of these sub-environments. \textit{It's important to emphasize that the interpretation of non-stationarity serves primarily as theoretical insight. In practical algorithm designs, we leverage this property in an approximate manner and do not require the environment to be explicitly decomposed into sub-environments.}

When considering a sample from the $k$-th sub-environment, the defined masks ${\bm{C}}_{(k)}^{\cdot \shortrightarrow \cdot}$ are used to represent invariant relationships within that sub-environment, meaning the current mask ${\bm{C}}_\text{current}^{\cdot \shortrightarrow \cdot}={\bm{C}}_{(k)}^{\cdot \shortrightarrow \cdot}$. The relationships among variables such as states, actions, and hidden states, are characterized by DAGs $\mathcal{D}_{(k)} = (V, E_{(k)})$ over the same node set $V=\{{\bm{s}}, {\bm{h}}, {\bm{a}}, {\bm{s}}^{\prime}, {\bm{h}}^{\prime}, {\bm{a}}^{\prime}, r, e\}$. Here, $e$ represents the environment label with an in-degree of $0$, and $E_{(k)}$ is the set of directed edges in each DAG. Edges in $E_{(k)}$ are determined by corresponding masks, \ie, an edge from node $v_i$ to $v_j (v_i, v_j \in V)$ exists in $E_{(k)}$ if and only if ${\bm{c}}_{(k)}^{v_i \shortrightarrow v_j} = 1$. Additionally, $E_{(k)}$ also includes edges from $e$ to $({\bm{s}}, {\bm{h}})$, which are subject to variation across different sub-environments. In this setting, non-stationarity is reflected by the changes in the underlying graph structure of these DAGs. 

To better understand the causal interpretation, we provide an example case involving two sub-environments, as shown in Figure \ref{fig:MAG}. This illustrates how the proposed union graph structure effectively captures the non-stationarity in the environment. For more detailed explanations, please refer to Appendix \ref{sec:full-theory}. Additionally, a toy example is provided in Appendix \ref{sec:toy-example} to further explain this concept.

\subsection{Union Graph of the Causal Structure}

In the context of the above causal interpretation, a significant challenge arises when learning the structure of a causal graph without access to the environment label $e$, which can be considered a latent confounder that leads to spurious correlations. In the presence of unobservable nodes, maximal ancestral graph (MAG) is a useful tool to generalize DAGs \citep{richardson2002ancestral}. For each environment-specific DAG $\mathcal{D}_{(k)}$, we can construct a corresponding MAG $\mathcal{M}_{(k)}$ \citep{sadeghi2013stable}, as outlined in Algorithm \ref{algo:MAG}. In these MAGs, bidirected edges ($\leftrightarrow$) are used to characterize the change of marginal distribution of ${\bm{s}}, {\bm{h}}$ over different sub-environments. To model structural relationships in non-stationary RL with a unified approach, we further encode the relations among all actions, states, hidden states, and rewards with an environment-shared union graph $\mathcal{M}_{\cup}$, as defined by Definition \ref{def:env} below.

\begin{definition}[Environment-shared union graph]\label{def:env}
The environment-shared union graph $\mathcal{M}_{\cup} := (V, D, B)$ has the set of nodes $V$,  the set of directed edges
\[
D = \{u\rightarrow v:u,v \in V, \exists k \text{ such that } u \rightarrow v \text{ in } \mathcal{M}_{(k)}\},
\]
and the set of bidirected edges
\[
B = \{u\leftrightarrow v:u,v \in V, \exists k \text{ such that } u \leftrightarrow v \text{ in } \mathcal{M}_{(k)}\}.
\]
\end{definition}

The above defined union graph $\mathcal{M}_{\cup}$ contains no cycle because Equation (\ref{eq:transition}-\ref{eq:reward-function}) implies that there exists a common topological ordering for $\left\{\mathcal{D}_{(k)}\right\}$, see Appendix \ref{sec:full-theory} for details.
Without knowing the label of the $k$-th sub-environment, we cannot generally identify the structure of $\mathcal{M}_{(k)}$ for each $k$ from the observed data. However, we can show that $\mathcal{M}_{\cup}$ is still a MAG, hence any non-adjacent pair of nodes is $d$-separated given some subset of nodes. 

\begin{proposition}\label{thm:idn}
Suppose that the dynamics follows Equation (\ref{eq:transition}-\ref{eq:reward-function}), 
then there exists a partial order $\pi$ on $V$ such that (a) $u$ is an ancestor of $v$ $\Rightarrow u<_\pi v \text{ in } \mathcal{M}_{(k)}$; and (b) $ u \leftrightarrow v \Rightarrow u \not\lessgtr_\pi v  \text{ in } \mathcal{M}_{(k)}$. As a consequence, the environment-shared union graph $\mathcal{M}_{\cup}$ is a MAG.
\end{proposition}

We provide the complete proofs and a detailed explanation of conditions in Appendix \ref{sec:full-theory}. Proposition \ref{thm:idn} provides theoretical support for recovering the structure of the environment-shared union graph $\mathcal{M}_{\cup}$. In the following sections, we will describe how the COREP algorithm learns a policy that is stable under non-stationarity by utilizing the environment-shared representation union graph $\mathcal{M}_{\cup}$.

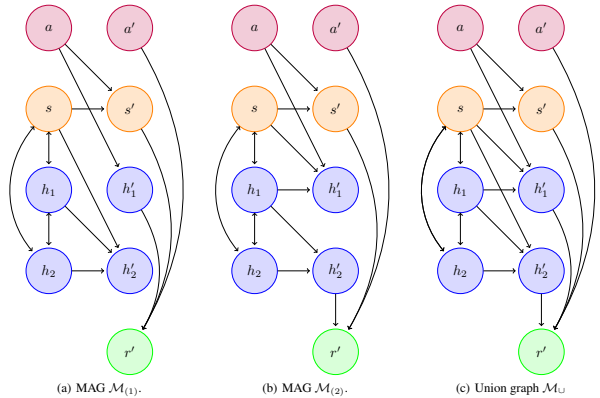
\begin{figure}[t]
\centering
\resizebox{\linewidth}{!}{
\subfigure[MAG $\mathcal{M}_{(1)}$.]{
\begin{tikzpicture}[scale=1,
->,
shorten >=2pt]
\node[circle,
minimum width = 32pt, draw=purple, fill=purple!20] (1) at(0,6){$a$};
\node[circle,
minimum width = 32pt, draw=orange, fill=orange!20] (2) at(0,4){$s$};
\node[circle,
minimum width =32pt ,draw=blue, fill=blue!15] (3) at(0,2){$h_{1}$};
\node[circle,
minimum width =32pt ,draw=blue, fill=blue!15] (4) at(0,0){$h_{2}$};
\node[circle,
minimum width =32pt,draw=purple, fill=purple!20] (5) at(2,6){$a^{\prime}$};
\node[circle,
minimum width =32pt ,draw=orange, fill=orange!20] (6) at(2,4){$s^{\prime}$};
\node[circle,
minimum width =32pt ,draw=blue, fill=blue!15] (7) at(2,2){$h_{1}^{\prime}$};
\node[circle,
minimum width =32pt ,draw=blue, fill=blue!15] (8) at(2,0){$h_{2}^{\prime}$};
\node[circle,
minimum width =32pt ,draw=green, fill=green!15] (9) at(2,-2){$r^{\prime}$};
\draw[<->] (2) --(3);
\draw[<->] (3) --(4);
\draw[<->] (2) to [out=230,in=130] (4);
\draw[->] (1) --(6);
\draw[->] (1) --(7);
\draw[->] (2) --(6);
\draw[->] (2) --(8);
\draw[->] (3) --(8);
\draw[->] (4) --(8);
\draw[->] (5) to [out=300,in=60] (9);
\draw[->] (6) to [out=300,in=60] (9);
\draw[->] (7) to [out=300,in=60] (9);
\end{tikzpicture}
}
\subfigure[MAG $\mathcal{M}_{(2)}$.]{
\centering
\begin{tikzpicture}[scale=1,
->,
shorten >=2pt]
\node[circle,
minimum width = 32pt, draw=purple, fill=purple!20] (1) at(0,6){$a$};
\node[circle,
minimum width = 32pt, draw=orange, fill=orange!20] (2) at(0,4){$s$};
\node[circle,
minimum width =32pt ,draw=blue, fill=blue!15] (3) at(0,2){$h_{1}$};
\node[circle,
minimum width =32pt ,draw=blue, fill=blue!15] (4) at(0,0){$h_{2}$};
\node[circle,
minimum width =32pt,draw=purple, fill=purple!20] (5) at(2,6){$a^{\prime}$};
\node[circle,
minimum width =32pt ,draw=orange, fill=orange!20] (6) at(2,4){$s^{\prime}$};
\node[circle,
minimum width =32pt ,draw=blue, fill=blue!15] (7) at(2,2){$h_{1}^{\prime}$};
\node[circle,
minimum width =32pt ,draw=blue, fill=blue!15] (8) at(2,0){$h_{2}^{\prime}$};
\node[circle,
minimum width =32pt ,draw=green, fill=green!15] (9) at(2,-2){$r^{\prime}$};
\draw[<->] (2) --(3);
\draw[<->] (3) --(4);
\draw[<->] (2) to [out=230,in=130] (4);
\draw[->] (1) --(6);
\draw[->] (1) --(7);
\draw[->] (2) --(6);
\draw[->] (2) --(7);
\draw[->] (3) --(7);
\draw[->] (3) --(8);
\draw[->] (4) --(8);
\draw[->] (5) to [out=300,in=60] (9);
\draw[->] (6) to [out=300,in=60] (9);
\draw[->] (8) -- (9);
\end{tikzpicture}
}
\subfigure[Union graph $\mathcal{M}_{\cup}$]{
\begin{tikzpicture}[scale=1,
->,
shorten >=2pt]
\node[circle,
minimum width = 32pt, draw=purple, fill=purple!20] (1) at(0,6){$a$};
\node[circle,
minimum width = 32pt, draw=orange, fill=orange!20] (2) at(0,4){$s$};
\node[circle,
minimum width =32pt ,draw=blue, fill=blue!15] (3) at(0,2){$h_{1}$};
\node[circle,
minimum width =32pt ,draw=blue, fill=blue!15] (4) at(0,0){$h_{2}$};
\node[circle,
minimum width =32pt,draw=purple, fill=purple!20] (5) at(2,6){$a^{\prime}$};
\node[circle,
minimum width =32pt ,draw=orange, fill=orange!20] (6) at(2,4){$s^{\prime}$};
\node[circle,
minimum width =32pt ,draw=blue, fill=blue!15] (7) at(2,2){$h_{1}^{\prime}$};
\node[circle,
minimum width =32pt ,draw=blue, fill=blue!15] (8) at(2,0){$h_{2}^{\prime}$};
\node[circle,
minimum width =32pt ,draw=green, fill=green!15] (9) at(2,-2){$r^{\prime}$};
\draw[<->] (2) --(3);
\draw[<->] (3) --(4);
\draw[<->] (2) to [out=230,in=130] (4);
\draw[->] (1) --(6);
\draw[->] (1) --(7);
\draw[->] (2) --(6);
\draw[->] (2) --(7);
\draw[->] (3) --(7);
\draw[->] (3) --(8);
\draw[->] (4) --(8);
\draw[->] (5) to [out=300,in=60] (9);
\draw[->] (6) to [out=300,in=60] (9);
\draw[->] (8) -- (9);
\draw[<->] (2) to [out=230,in=130] (4);
\draw[->] (2) --(8);
\draw[->] (7) to [out=300,in=60] (9);
\end{tikzpicture}
}
}
\caption{MAG representations for two sub-environments and their union graph. In this example, the union graph is capable of representing all possible kinds of causal relationships within the changing dynamics. More explanations can be found in Appendix \ref{sec:full-theory}.}
\label{fig:MAG}
\end{figure}

\subsection{Dual Graph Attention Network Structure}\label{sec:gat-structure}

\begin{figure*}[t]
    \centering
    \includegraphics[width=.98\textwidth]{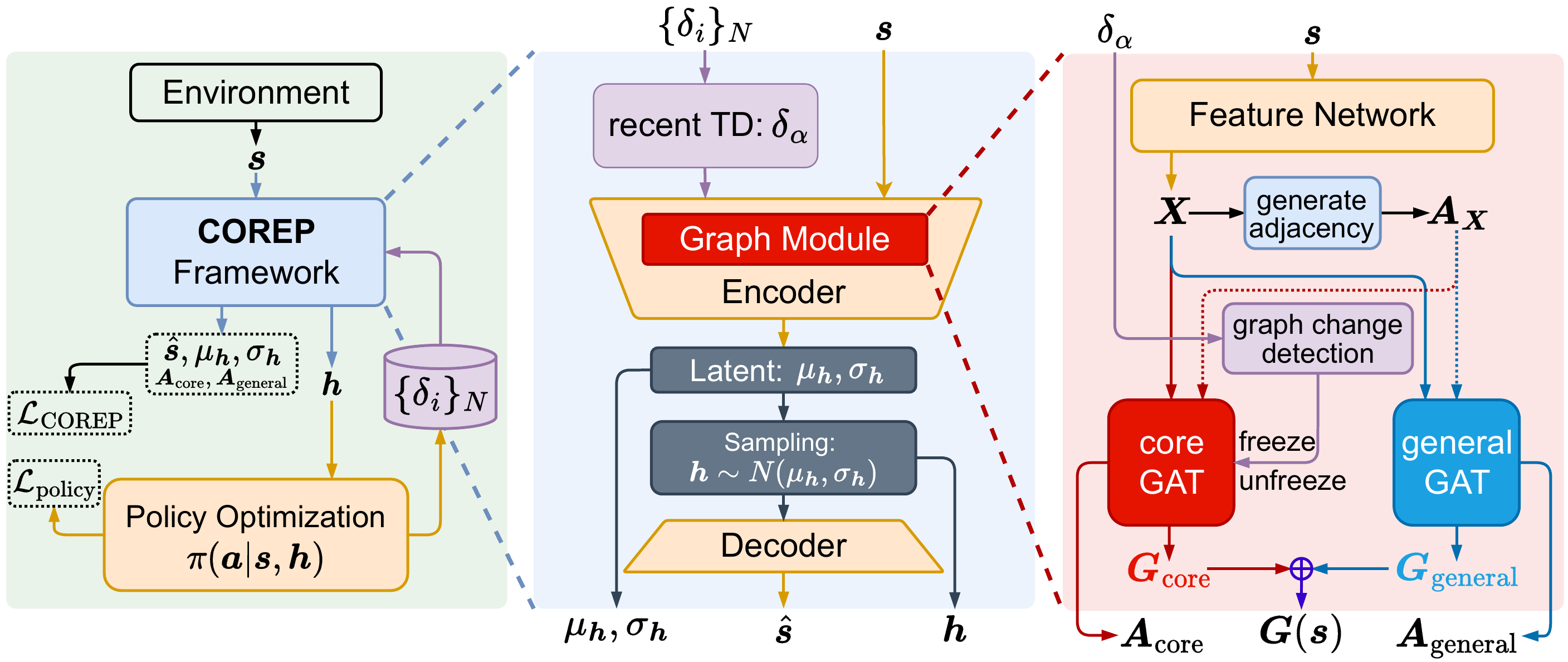}
    \caption{Overview of the COREP framework. (1) The left part illustrates that the COREP framework can be seamlessly incorporated into any RL algorithm. It takes the state as input and outputs the causal-origin representation for policy optimization. (2) The middle part shows the VAE structure employed by the COREP framework, which is utilized to enhance the learning efficiency. (3) The right part highlights the key components of COREP. The dual GAT structure is designed in line with the concept of causal-origin representation to retain the essential parts of the graph. The TD error detection can guide the core-GAT to learn the environment-shared union graph based on our theory. The general-GAT is continuously updated to compensate for the potential loss of information.}
    \label{fig:COR-framework}

\end{figure*}

In this section, we will focus on the structure design. The detailed update mechanism will be discussed in Section \ref{sec:updating}.

In line with Proposition \ref{thm:idn}, our objective is to efficiently learn the causal-origin representation, which encapsulates the environment-shared union graph $\mathcal{M}_{\cup}$. To achieve this, we design a dual GAT structure comprising a core-GAT and a general-GAT. The core-GAT is specifically designed to learn a stable graph representation that aligns with the environment-shared union graph. For this purpose, we use TD error as a simple yet effective detector for significant changes in the environment's underlying graph structure.
This approach enables selective updates to the core-GAT, ensuring it responds only to significant environmental changes, thereby eliminating the need for explicit recognition of such changes as often required by existing work \citep{sutton2007role}. Additionally, it facilitates the approximate learning of the environment-shared union graph.

While the core-GAT focuses on learning the stable part of the graph representation, some edges may be overlooked or lost in the process. To compensate for this potential loss of information and to enhance the algorithm's adaptation capabilities, we introduce a continuously updating general-GAT. By integrating both the core-GAT and the general-GAT, we can construct the causal-origin representation to provide a comprehensive understanding of the environment's dynamics for the policy and mitigating the impact of non-stationarity in RL.

Specifically, we first transform states into node features using an MLP network $f_\mathrm{MLP}:\mathbb{R}^{d_s}\to\mathbb{R}^{N\cdot d_f}$ and reshape the output into node feature matrix ${\bm{X}} = \{{\bm{x}}_1, {\bm{x}}_2,\ldots,{\bm{x}}_N\} \in \mathbb{R}^{N\times d_f}$, where $N$ is the number of nodes, and $d_f$ is the number of features in each node. We then compute the weighted adjacency matrix which represents the probabilities of edges by using Softmax on the similarity matrix of nodes: 
\begin{equation}\label{eq:compute-adjacency}
    {\bm{A_X}}=\mathrm{Softmax}\left({\bm{X}}{\bm{X}}^{\mathrm{T}}\odot({\bm{1}}_N-{\bm{I}}_N)\right),
\end{equation}
where ${\bm{1}}_N\in\mathbb{R}^{N\times N}$ represents the matrix with all elements equal to 1, ${\bm{I}}_N\in\mathbb{R}^{N\times N}$ represents the identity matrix, and $\odot$ denotes the Hadamard product. Multiplying $({\bm{1}}_N-{\bm{I}}_N)$ is for removing the self-loop similarity when computing the weighted adjacency matrix. 

Then a learnable weight matrix ${\bm{W}} \in \mathbb{R}^{d_f \times d_g}$ is applied to the nodes for transforming $\bm{X}$ into graph features $\bm{XW}\in\mathbb{R}^{N\times d_g}$ where $d_g$ denotes the dimension of the graph feature. We then perform the self-attention mechanism on the nodes: 
\begin{equation}
\alpha_{ij} = \operatorname{attention}({\bm{x}}_i\bm{W},{\bm{x}}_j\bm{W}|{\bm{A_X}}).
\end{equation}
The conditioned ${\bm{A_X}}$ allows us to perform the masked attention, \ie, we only compute $\alpha_{ij}$ for node $j\in \mathbf{N}_i({\bm{A_X}})$ where $\mathbf{N}_i({\bm{A_X}})$ is the neighbor set of node $i$ computed by the weighted adjacency matrix ${\bm{A_X}}$. This helps us consider deeper-depth neighbors of each node by combining multiple $\mathrm{attention(\cdot)}$ into a multi-layer network. For the $n$-th graph attention layer, the coefficients computed by the self-attention mechanism can be specifically expressed as:
\begin{equation}\label{eq:gat_layer_coef}
    \alpha_{i j}=\frac{\delta_{\mathbf{N}_i({\bm{A_X}})}(j)\cdot \exp \left(\sigma\left({\bm{l}_n}\left[{\bm{x}_i\bm{W}} \oplus {\bm{x}_j\bm{W}}\right]^\mathrm{T}\right)\right)}{\sum_{k \in \mathbf{N}_i({\bm{A_X}})} \exp \left(\sigma\left({\bm{l}_n}\left[{\bm{x}_i\bm{W}} \oplus {\bm{x}_k\bm{W}}\right]^{\mathrm{T}}\right)\right)},
\end{equation}
where $\oplus$ is the concatenation operator, $\sigma$ is the activation function, ${\bm{l}_n}\in\mathbb{R}^{2d_g}$ is the learnable weight for the $n$-th graph attention layer, and $\delta_{\mathbf{N}_i({\bm{A_X}})}(j)$ is the indicator function, \ie, $\delta_{\mathbf{N}_i({\bm{A_X}})}(j)=1$ if $j\in\mathbf{N}_i({\bm{A_X}})$ otherwise $0$. 
Subsequently, the resulting features are concatenated to form the graph node: 
\begin{equation}
{\bm{g}}_i=\sigma\Big(\sum_{j \in \mathbf{N}_i({\bm{A_X}})} \alpha_{i j} {\bm{x}_j\bm{W}}\Big).
\end{equation}
The respective outputs of the core-GAT and the general-GAT, \ie, ${\bm{G}}_\text{core}\doteq \left\{\bm{g}_1, \bm{g}_2, \ldots, \bm{g}_N\right\}_\text{core}^\mathrm{T}$, ${\bm{G}}_\text{general}\doteq \left\{\bm{g}_1, \bm{g}_2, \ldots, \bm{g}_N\right\}_\text{general}^\mathrm{T}$, are concatenated to form the final causal-origin representation. Specifically, we denote the entire process of obtaining the causal-origin representation from ${\bm{s}}$ as a function $\bm{G}:\mathbb{R}^{d_s}\to\mathbb{R}^{N\times 2d_g}$, such that $\bm{G}({\bm{s}}) \doteq {\bm{G}}_\text{core}\oplus {\bm{G}}_\text{general}$.

We finally feed $\bm{G}(\bm{s})$ into a VAE inference process to derive the latent representation ${\bm{h}}$. More details about this process are discussed in Section \ref{sec:vae}. 
The latent ${\bm{h}}$ is then provided to the policy $\pi({\bm{a}}|{\bm{s}},{\bm{h}})$ for policy optimization. COREP does not restrict the choice of policy optimization algorithms. In our implementation, we choose the standard PPO algorithm \citep{schulman2017proximal} for policy optimization.

\subsection{Incorporation with VAE}\label{sec:vae}

To improve the efficiency of learning the causal-origin representation, we incorporate the causal-origin representation into the Variational AutoEncoder (VAE) framework \citep{kingma2013auto}.  Specifically, we feed the output $\bm{G}(\bm{s})$ into the VAE inference process to derive the mean and variance $(\mu_{\bm{h}}, \sigma_{\bm{h}})$ of the latent representation ${\bm{h}}$. Subsequently, we can sample ${\bm{h}} \sim \mathcal{N}(\mu_{\bm{h}}, \sigma_{\bm{h}}), {\bm{h}}\in\mathbb{R}^{d_h}$. The loss function for VAE is defined as
\begin{equation}\label{eq:l-vae}
\begin{split}
&\mathcal{L}_\text{VAE}({\bm{s}}; \theta,\phi) \\
&= \mathbb{E}_{q_{\phi}({\bm{h}}|\bm{G}(\bm{s}))}\left[\log p_{\theta}(\bm{G}(\bm{s})|{\bm{h}})\right] - \mathrm{KL}\left[q_{\phi}({\bm{h}}|\bm{G}(\bm{s})) || p({\bm{h}})\right]\\
&\approx \mathrm{MSE}({\bm{s}}, \hat{{\bm{s}}}) - \mathrm{KL}\left[q_{\phi}\left({\bm{h}}|\bm{G}(\bm{s})\right) || \mathcal{N}(0,{\bm{I}})\right],
\end{split}
\end{equation}
where $p_\theta,q_\phi$ represent the parameterized decoder and encoder respectively, $\mathrm{KL}(\cdot)$ denotes the Kullback-Leibler divergence, and $\mathrm{MSE}({\bm{s}}, \hat{{\bm{s}}})$ is an estimation of $\mathbb{E}_{q_{\phi}({\bm{h}}|\bm{G}(\bm{s}))}\left[\log p_{\theta}(\bm{G}(\bm{s})|{\bm{h}})\right]$ which measures the mean square error between the original state and the reconstructed state with the causal-origin representation. \textit{It is noteworthy that the VAE structure serves solely as a tool to enhance learning efficiency, therefore it is not a strictly necessary component of our method}. The latent ${\bm{h}}$ is then provided to the policy $\pi({\bm{a}}|{\bm{s}},{\bm{h}})$ for policy optimization.

\subsection{Guided Updating for Core-GAT}\label{sec:updating}

As discussed in Section \ref{sec:gat-structure}, to learn the causal-origin representation encapsulating the environment-shared union graph $\mathcal{M}_{\cup}$, we design a TD error-based detection mechanism to guide the update of core-GAT. 
Specifically, we store the TD errors of policy optimization into a buffer and compute the mean of recent TD errors, denoted as $\delta_\alpha$. Here, $\alpha$ controls the proportion of recent TD errors for detection.
We then check whether $\delta_\alpha$ lies within the confidence interval $(\mu_{\delta}-\eta\sigma_{\delta}, \mu_{\delta}+\eta\sigma_{\delta})$, where $\mu_\delta,\sigma_\delta$ are the mean and standard deviation of the TD buffer, and $\eta$ represents the confidence level. If the recent TD error $\delta_\alpha$ lies within this interval, we freeze the weights of the core-GAT and halt its updates; otherwise, we unfreeze its weights and proceed with updating the core-GAT.

We further introduce a regularization that penalizes the difference between the output adjacency matrices of the two GATs to guide the learning of the core-GAT:
\begin{equation}\label{eq:l-guide}
    \mathcal{L}_{\text {guide}} = \|{\bm{A}}_{\text{core}}-{\bm{A}}_{\text{general}}\|_2.
\end{equation}

\begin{figure*}[t]
    \centering
    \includegraphics[width=.98\linewidth]{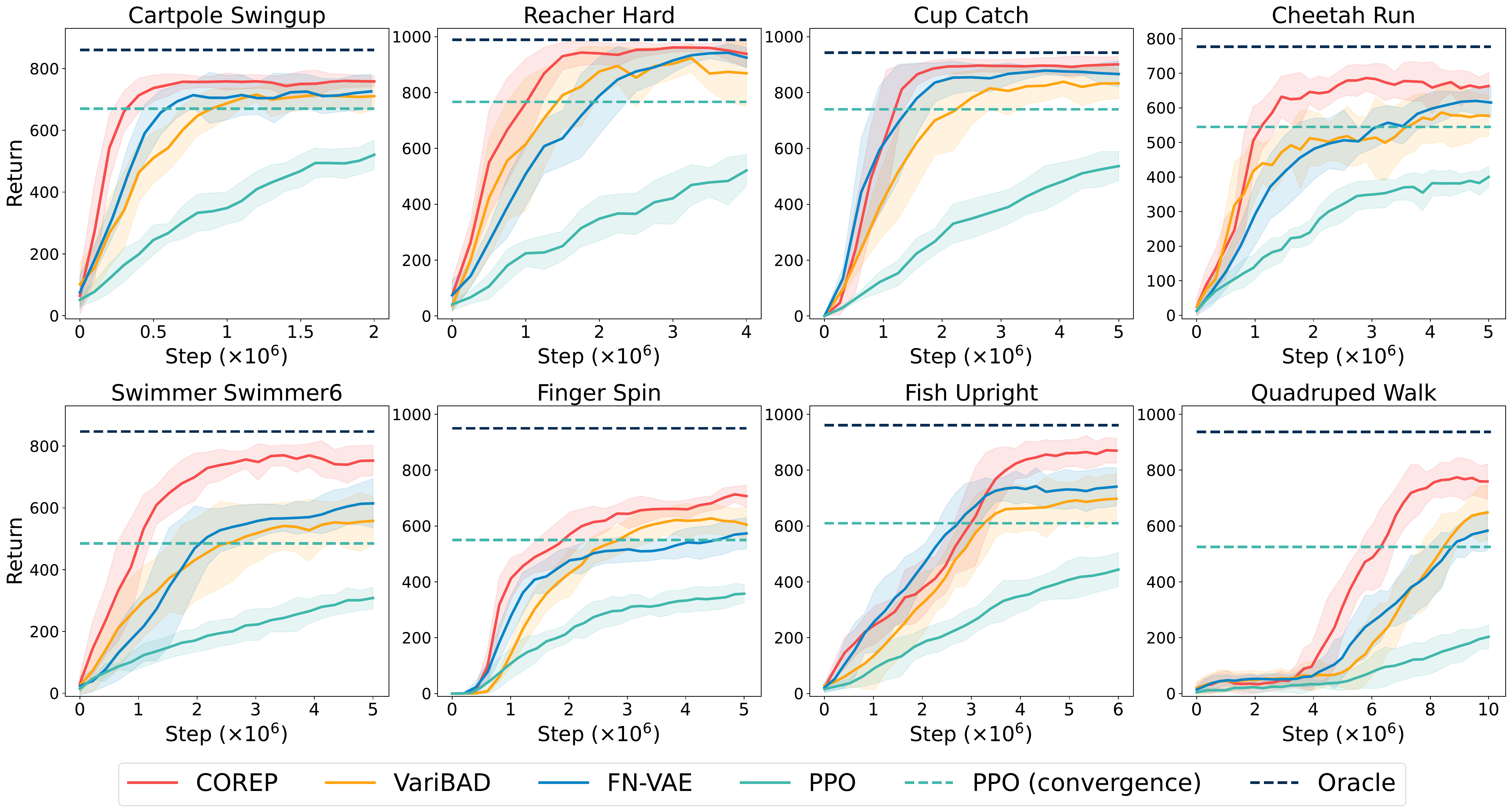}
    \caption{Learning curves of our COREP algorithm and other baselines in different tasks. Solid curves indicate the mean of all trials with 5 different seeds. Shaded regions correspond to the standard deviation among trials. Dashed lines represent the asymptotic performance of PPO and Oracle.}
    \label{fig:multi-env-perf}
\end{figure*}

To enhance the identifiability, we introduce the regularization for the MAG structure and sparsity:
\begin{equation}\label{eq:l-sparsity-mag}
\begin{gathered}
\mathcal{L}_\text{MAG} = \|{\bm{A}}_{\text{core}} - {\bm{A}}_{\text{core}}^{\mathrm{T}}\|_2 + \|{\bm{A}}_{\text{general}} - {\bm{A}}_{\text{general}}^{\mathrm{T}}\|_2,\\
\mathcal{L}_{\text {sparsity}} = \|{\bm{A}}_{\text{core}}\|_1 + \|{\bm{A}}_{\text{general}}\|_1.
\end{gathered}
\end{equation}

The $\mathcal{L}_\text{MAG}$ can penalize asymmetry in the adjacency matrices of both the core and general graphs, ensuring alignment with the design of MAGs which utilize bidirected edges to characterize the change of marginal distribution of ${\bm{s}},  {\bm{h}}$ across different sub-environments. The $\mathcal{L}_{\text {sparsity}}$ can encourage a sparse representation, which is crucial for maintaining a manageable and interpretable graph structure. 

We finally compute the total loss function $\mathcal{L}_\text{total}$ as shown in Equation (\ref{eq:objective}), combining the policy optimization objective $\mathcal{L}_\text{policy}$ with the regularization loss functions:
\begin{equation}\label{eq:objective}
    \mathcal{L}_\text{total} =\mathcal{L}_\text{policy} + \lambda_1 \mathcal{L}_\text{guide} + \lambda_2(\mathcal{L}_\text{MAG} + \mathcal{L}_\text{sparsity} + \mathcal{L}_\text{VAE}),
\end{equation}
where $\mathcal{L}_\text{VAE}$ is calculated according to Equation (\ref{eq:l-vae}). It is important to note that, despite the presence of multiple regularization terms, we simplify the objective by aggregating these terms based on their magnitudes, setting only two hyperparameters. Empirical results also demonstrate that COREP does not rely on complex parameter tuning.

By computing and backpropagating the gradient of $\mathcal{L}_\text{total}$, along with the policy optimization in an end-to-end manner, COREP prevents the distinct loss functions from leading to irregular causal structures and ensures that the policy learned can effectively tackle non-stationarity problems in RL.  The detailed steps of COREP are outlined in Algorithm \ref{algo:COR}. Implementation details are shown in Appendix \ref{sec:implement-details}.

\section{Experiments}

In this section, our objective is to thoroughly evaluate the COREP algorithm by addressing three key questions:
(1) How effective is COREP in tackling non-stationarity? (2) What specific contribution does each component of COREP make to its overall performance? (3) Can COREP maintain consistent performance across various degrees and settings of non-stationarity? To answer these questions, we conduct various experiments and provide corresponding analyses. Due to the page limitation, only parts of these experiments are presented in the main manuscript. For complete results, further analysis, and implementation details, please refer to Appendix \ref{sec:implement-details} and \ref{sec:full-exp}. To ensure reproducibility, we include our code in the supplementary material.

\textbf{Baselines.} We compare COREP with the following baselines: FN-VAE \citep{feng2022factored}, VariBAD \citep{zintgraf2019varibad}, and PPO \citep{schulman2017proximal}. FN-VAE is the SOTA method for tackling non-stationarity, VariBAD is one of the SOTA algorithms in meta-RL that also has certain capabilities in handling non-stationarity, and PPO is a classical algorithm known for its strong stability. Furthermore, to examine the performance degradation caused by non-stationarity, we include an Oracle that has full access to non-stationarity information.

\textbf{Environment settings.} The experiments are conducted on various environments from the DeepMind Control Suite \citep{tassa2018deepmind}, which is a widely used benchmark for RL algorithms. We modify the environments to introduce non-stationarity, enabling a comprehensive evaluation of COREP. 
In our settings, similar to prior work FN-VAE, we introduce periodic noises to the environmental dynamics to represent non-stationarity. To support our claim that COREP can handle more complex non-stationarity, we design a more intricate setting, \ie, we randomly sample the coefficients for both within-episode and across-episode non-stationarity at every time step. Specifically, our modification can be expressed as:
\begin{equation}\label{eq:non-stat}
   s^\prime = f(s,a) + f(s,a)\cdot\alpha_d\left[c^t_1\cos(c_2^t\cdot t)+c_3^i\sin(c_4^i\cdot i)\right].
\end{equation}
Here, $\alpha_d$ controls the overall degree of non-stationarity, and $c_k^t, c_k^i\sim \mathcal{N}(0.5, 0.5)$ represent the changing coefficients of \textit{within-episode} and \textit{across-episode} non-stationarity, respectively. This design generates various combinations of non-stationarity for each time step and episode, posing more significant challenges to our COREP algorithm and the compared baselines.

\textbf{Performance.} As illustrated in Figure \ref{fig:multi-env-perf}, COREP outperforms all baselines across various environments, showcasing consistent performances in the face of non-stationarity. Notably, in complex environments such as \textit{Swimmer Swimmer6, Fish Upright, and Quadruped Walk}, COREP not only achieves higher performance but also exhibits smaller variances. This indicates its strong resilience and adaptability to intricate non-stationary environments. 

Comparatively, VariBAD demonstrates certain resistance to non-stationarity due to its adaptive capabilities. However, the large variances indicate a lack of stability in non-stationary settings. The FN-VAE method, which explicitly models the change factors, shows competitive performance in simpler environments but struggles to maintain consistency in more complex scenarios, underscoring its limitations in handling more challenging non-stationarity. Furthermore, the narrow performance gap between COREP and Oracle indicates the effectiveness of COREP in reducing the impact of non-stationarity on performance.

\textbf{Ablation study.} We conduct ablation studies to analyze the contribution of each component in COREP. To ensure consistency in our conclusion, the experiments are conducted under various non-stationarity settings: \textit{within-episode \& across-episode} (labeled as W+A-EP), \textit{within-episode} (W-EP), and \textit{across-episode} (A-EP) non-stationarities. 

As depicted in Figure \ref{fig:ablation}, removing all COREP-specific designs and retaining only the VAE process (the `w/o COREP' variant) results in substantial decreases in performance. This highlights the overall effectiveness of our designs. 

In the single GAT variant, which maintains the same network structure without introducing a secondary GAT and corresponding update mechanisms, we observe a significant performance gap compared to the full COREP. This result indicates that merely incorporating a graph representation is insufficient for capturing non-stationarity information effectively. 

The removal of the TD error detection mechanism (the `w/o $\mathcal{L}_\text{guide}$' variant) lead to considerable performance drops, further substantiating the importance of the guided update mechanism in COREP. 

These results demonstrate the effectiveness and necessity of the two key designs in our method (\ie, the dual GAT structure and the guided update mechanism) in tackling non-stationarity. Additionally, the results also show that the regularization terms $\mathcal{L}_\text{MAG}$ and $\mathcal{L}_\text{sparse}$ play important roles in enhancing COREP’s ability to handle non-stationarity. Furthermore, the inclusion of the VAE process in COREP is found to be effective in improving the performance as expected.

\begin{figure}[t]
\centering
\subfigure[Ablation study.]{
\label{fig:ablation}
{\includegraphics[width=\linewidth]{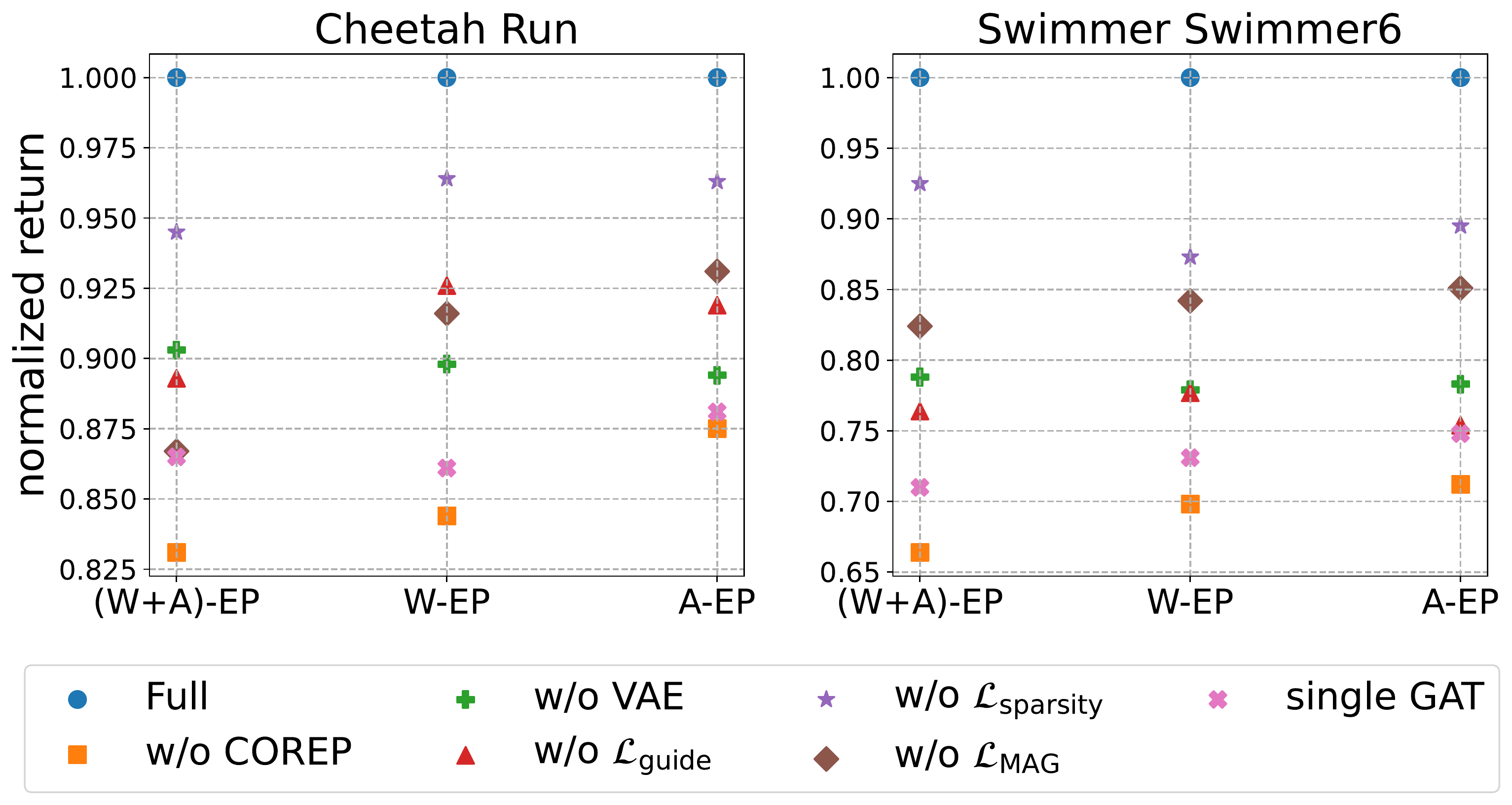}}
}
\hfill
\subfigure[Different degrees of non-stationarity.]{
\label{fig:non-stat-degree}
{\includegraphics[width=\linewidth]{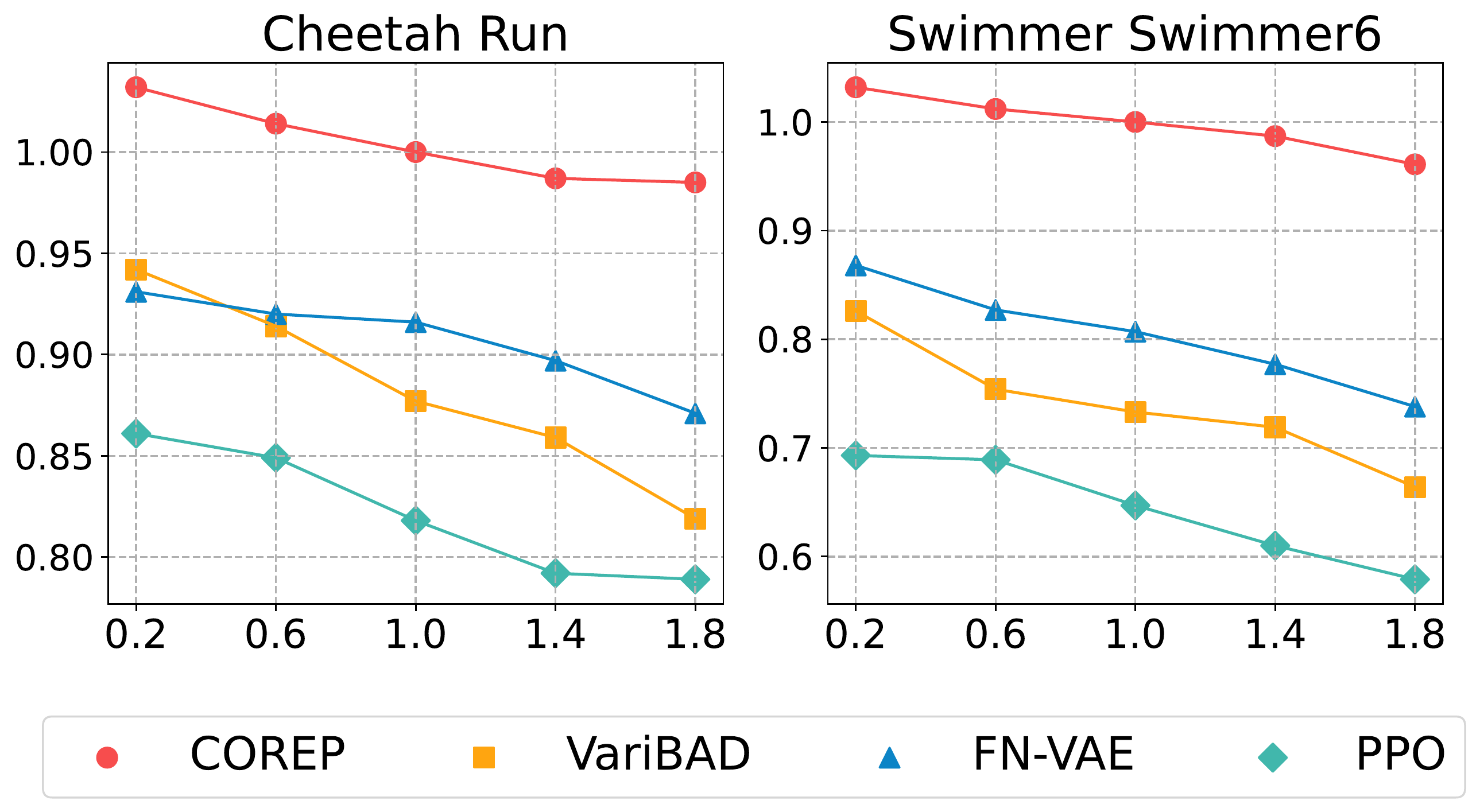}}
}
\caption{Mean returns of 3 different trials with: (a) different components and non-stationarity settings. Returns are normalized to the full version of COREP in each environment; (b) different degrees of non-stationarity. Returns are normalized to the COREP algorithm with standard degree $1.0$.}
\label{fig:ablation_and_degree}
\end{figure}

\textbf{Different degrees of non-stationarity.} 

We further investigate the impact of varying degrees of non-stationarity in the environment, as depicted in Figure \ref{fig:non-stat-degree}. The results suggest that the compared baselines are more affected by the degree of non-stationarity. Contrastively, COREP exhibits consistent performance when encountering different degrees of non-stationarity, further demonstrating our claim that COREP can effectively tackle more complex non-stationarity.

\textbf{Visualization.} To visualize the graph structure learned by COREP, we respectively show the weighted adjacency matrix of core-GAT and general-GAT in \textit{Cheetah Run} and \textit{Walker Walk}. Please refer to Section \ref{sec:vis} for more details. It can be seen that core-GAT focuses more on a few core nodes in its learned graph structure, while general-GAT compensates for some overlooked detailed information by core-GAT. The results align well with our claim made in Section \ref{sec:idea}.

\textbf{Hyperparameter study.} We conduct additional experiments to study COREP's sensitivity to the hyperparameters $\lambda_1, \lambda_2$ in the objective function (\ref{eq:objective}). Results are shown in Appendix \ref{sec:tuning}, indicating that the performance of COREP is not sensitive to the hyperparameters. Even with extensive adjustments to $\lambda_1, \lambda_2$, COREP consistently surpasses the SOTA baseline. Notably, we observe instances in certain environments with higher performance than default. This suggests that with more precise tuning, there is potential to further enhance COREP's effectiveness. These findings suggest that COREP's superior performance is largely attributed to its designs rather than relying on hyperparameter tuning.

\section{Related Work}

\textbf{Non-stationary RL.} Pioneering research in non-stationary RL primarily focused on directly detecting changes that had already occurred \citep{da2006dealing}, rather than anticipating them. Various methods have been developed to anticipate changes in non-stationary RL settings. For example, Prognosticator \citep{chandak2020optimizing} tried to maximize future rewards without explicitly modeling non-stationary environments, while MBCD \citep{alegre2021minimum} employed change-point detection to determine whether an agent should learn a new policy or reuse existing ones. However, these methods with change-point detection may not work well in complex non-stationary environments and often requires providing priors. 

In cases where the evolution of non-stationary environments can be represented as a Semi-Markov chain, Hidden Markov-MDPs or Hierarchical Semi-Markov Decision Processes can be employed to address non-stationarity problems \citep{choi1999environment, hadoux2014solving}. Some later work attempts to resist non-stationarity by leveraging the generalization of meta-learning \citep{finn2017model}. For example, Adaption via Meta-learning \citep{al2017continuous} integrated continuous adaptation into the learning-to-learn framework to solve the non-stationarity problems. TRIO \citep{poiani2021meta} tracked non-stationarity by inferring the evolution of latent parameters, capturing the temporal change factors during the meta-testing phase. GrBAL \citep{nagabandi2018learning} meta-trained dynamic priors, enabling efficient adaptation to local contexts.

However, these methods require the pre-definition of non-stationary tasks and subsequent meta-training on them. In real-world scenarios, though, we cannot access such information about the non-stationarity. An alternative line of research directly tries to learn latent representations to capture non-stationary components, leveraging latent variable models to directly model change factors in environments or estimating latent vectors describing the non-stationary aspects of dynamics. 

Specifically, LILAC \citep{xie2020deep} regarded the change factor as a latent variable and explicitly modeled the latent MDP. FN-VAE \citep{feng2022factored} modeled multiple latent variables of non-stationarities to achieve better performance. However, in real-world scenarios, non-stationarity itself is often more complex. Simply modeling the latent dynamics may not solve such complex scenarios well. Interpreting nonstationarity from a causal perspective is another novel approach. 

In addition, some work learns controllers on a collection of pre-defined stationary environments \citep{provan2022towards, deng2022towards, zhang2019stable}, which can get a guaranteed controller for any mixture of these stationary environments, thereby improving performance in more complex environments. Although our method theoretically treats non-stationary environments as mixtures of sub-environments in a similar way, we use a more elegant update mechanism in practical design to ensure that our method is applicable to continuously changing environments. 

Some other research \citep{saengkyongam2023invariant} seeks to identify an invariant causal structure to mitigate the impact of non-stationarity, presenting similarities to our approach. However, their methodology relies on offline data and has been tested solely in simple contextual Bandits environments. In contrast, our COREP algorithm is capable of online learning and addresses non-stationarity in more complex environments.

\textbf{Causal structure learning.}
Various approaches for learning causal structure from observed data have been proposed, see \citep{vowels2022d} for a review. These approaches mainly fall into two broad categories: constraint-based methods and score-based methods. The constraint-based methods check the existence of edges by performing conditional independence tests between each pair of variables, \eg PC \citep{spirtes2000causation}, IC \citep{pearl2000Causality}, and FCI \citep{spirtes1995causal,zhang2008completeness}. In contrast, score-based methods generally view causal structure learning as a combinatorial optimization problem, and measure the goodness of fit of graphs over the data with a score, then optimize such score to find an optimal graph or equivalent classes  \citep{chickering2002optimal, koivisto2004exact,silander2006simple,cussens2017polyhedral, huang2018generalized}.  

Recently, some gradient-based methods that transform the discrete search into a continuous optimization by relaxing the space over DAGs have been proposed. These methods allow for applying continuous optimizations such as gradient descent to causal structure learning. For example, NOTEARS \citep{zheng2018dags} reformulated the structure learning problem as a continuous optimization problem, and ensured acyclicity with a weighted adjacency matrix. DAG-GNN \citep{yu2019dag} proposed a generative model parameterized by a GNN and applied a variant of the structural constraint to learn the DAG. Some researchers \citep{saeed2020causal} considered the distribution arising from a mixture of causal DAGs, used MAGs to represent DAGs with unobserved nodes, and showed the identifiability of the union of component MAGs.

\section{Conclusions, Limitations and Future Work}

In this work, we first offer a novel interpretation of non-stationarity in RL, characterizing it through the union of MAGs. This new perspective has inspired us to design the COREP algorithm, which features a dual GAT structure and an update mechanism guided by TD-error detection. Focusing on the causal relationships whthin the dynamics, COREP learns a causal-origin representation that remains stable amidst changes in the environment, effectively addressing non-stationarity problems in RL. Our theoretical analysis offers both inspiration and foundational support for COREP. Furthermore, experimental results from various non-stationary environments demonstrate the efficacy of our algorithm.

However, the COREP algorithm does face certain limitations, particularly scalability issues in high-dimensional state spaces due to the computationally intensive nature of its graph-based representation. In our future work, we aim to overcome these challenges by integrating the causal-origin representation with other types of latent variable models, such as normalizing flows and probabilistic graphical models. This is expected to enhance both the scalability and the performance, making COREP more applicable in real-world scenarios.

\newpage

\section*{Acknowledgements}

This work was supported in part by NSF China under grant 62250068. The authors would like to thank the anonymous
reviewers for their valuable feedback and suggestions.

\section*{Impact Statement}

This research in machine learning, especially through the COREP algorithm, offers potential benefits for enhancing the decision-making abilities of agents in complex and unpredictable environments, such as robot control and autonomous driving. However, it is crucial to be aware of ethical implications, including data privacy, algorithmic bias, and the security of these systems. Future research should focus on developing transparent and safe algorithms for real-world scenarios.


\bibliography{paper}
\bibliographystyle{icml2024}

\newpage
\appendix
\onecolumn
\section{Causality Background and Proofs}\label{sec:full-theory}

We first review the definition of the Markov condition, the faithfulness assumption and some graphical concepts shown in the condition of Theorem \ref{thm:idn}.
We use $\mathrm{pa}_{\mathcal{D}}(v)$, $\mathrm{ch}_{\mathcal{D}}(v)$, and $\mathrm{an}_{\mathcal{D}}(v)$ to denote the parents, children and ancestor of node $v$, respectively; for the detailed definitions, see \eg \citep{lauritzen1996graphical}.

\begin{definition}[Global Markov Condition \citep{pearl2000Causality}]\label{assumption:markov}
A distribution $P$ over $V$ satisfies the global Markov condition on graph $\mathcal{D}$ if for any partition $(X, Y, Z)$ such that $X$ is d-separated from $Y$ given $Z$, then $X$ and $Y$ are conditionally independent given $Z$. 
\end{definition}

\begin{definition}[Faithfulness \citep{pearl2000Causality}]\label{assumption:faithful}
There are no independencies between variables that are not entailed by the Markov Condition.
\end{definition} 

Under the above assumptions, we can tell the conditional independences using the d-separation criterion from a given DAG $\mathcal{D}$ \citep{pearl2000Causality}. Similarly, the MAGs are ancestral graphs where any non-adjacent pair of nodes is d-separated \citep{richardson2002ancestral}.
The following algorithm shows how to construct a MAG from DAG \citep{saeed2020causal}:

\begin{algorithm}
\caption{Construction of the maximal ancestral graph}\label{algo:MAG}
\begin{algorithmic}[1]
\STATE Input: DAG ${\mathcal{D}} = (V, E)$
\STATE Initialize $D=\emptyset, B=\emptyset$
\FOR{$u,v\in\mathrm{ch}_{\mathcal{D}}(y)$} 
\STATE add $u\leftrightarrow v$ to $B$.
\ENDFOR
\FOR{$t,u,v$ such that $(t\rightarrow u)\in E$ and $(u\leftrightarrow v)\in B$} 
\IF{$u\in \mathrm{an}_{\mathcal{D}}(v)$}
    \STATE add $t\rightarrow v$ to $D$
\ENDIF
\ENDFOR
\FOR{$u,v$ such that $(u\leftrightarrow v)\in B$} 
\IF{$u \in \mathrm{an}_{\mathcal{D}}(v)$}
    \STATE remove $u\leftrightarrow v$ from $B$ and add $u\rightarrow v$ to $D$
\ENDIF
\ENDFOR
\end{algorithmic}
\end{algorithm}

To illustrate the above algorithm, we provide two figures. Figure \ref{fig:G} shows the underlying causal DAGs for the two environments, and Figure \ref{fig:M} depicts the output of Algorithm \ref{algo:MAG} as well as the corresponding environment-shared union graph.

\begin{figure}[ht]
\centering
\resizebox{.5\linewidth}{!}{
\subfigure[DAG $\mathcal{D}_{(1)}$.]{
\begin{tikzpicture}[scale=0.9,
->,
shorten >=2pt]
\node[circle,
minimum width = 32pt, draw=yellow, fill=yellow!20] (0) at(-2,2){$e$};
\node[circle,
minimum width = 32pt, draw=purple, fill=purple!20] (1) at(0,6){$a$};
\node[circle,
minimum width = 32pt, draw=orange, fill=orange!20] (2) at(0,4){$s$};
\node[circle,
minimum width =32pt ,draw=blue, fill=blue!15] (3) at(0,2){$h_{1}$};
\node[circle,
minimum width =32pt ,draw=blue, fill=blue!15] (4) at(0,0){$h_{2}$};
\node[circle,
minimum width =32pt,draw=purple, fill=purple!20] (5) at(2,6){$a^{\prime}$};
\node[circle,
minimum width =32pt ,draw=orange, fill=orange!20] (6) at(2,4){$s^{\prime}$};
\node[circle,
minimum width =32pt ,draw=blue, fill=blue!15] (7) at(2,2){$h_{1}^{\prime}$};
\node[circle,
minimum width =32pt ,draw=blue, fill=blue!15] (8) at(2,0){$h_{2}^{\prime}$};
\node[circle,
minimum width =32pt ,draw=green, fill=green!15] (9) at(2,-2){$r^{\prime}$};
\draw[->] (0) --(2);
\draw[->] (0) --(3);
\draw[->] (0) --(4);
\draw[->] (1) --(6);
\draw[->] (1) --(7);
\draw[->] (2) --(6);
\draw[->] (2) --(8);
\draw[->] (3) --(8);
\draw[->] (4) --(8);
\draw[->] (5) to [out=300,in=60] (9);
\draw[->] (6) to [out=300,in=60] (9);
\draw[->] (7) to [out=300,in=60] (9);
\end{tikzpicture}
}
\subfigure[DAG $\mathcal{D}_{(2)}$.]{
\centering
\begin{tikzpicture}[scale=0.9,
->,
shorten >=2pt]
\node[circle,
minimum width = 32pt, draw=yellow, fill=yellow!20] (0) at(-2,2){$e$};
\node[circle,
minimum width = 32pt, draw=purple, fill=purple!20] (1) at(0,6){$a$};
\node[circle,
minimum width = 32pt, draw=orange, fill=orange!20] (2) at(0,4){$s$};
\node[circle,
minimum width =32pt ,draw=blue, fill=blue!15] (3) at(0,2){$h_{1}$};
\node[circle,
minimum width =32pt ,draw=blue, fill=blue!15] (4) at(0,0){$h_{2}$};
\node[circle,
minimum width =32pt,draw=purple, fill=purple!20] (5) at(2,6){$a^{\prime}$};
\node[circle,
minimum width =32pt ,draw=orange, fill=orange!20] (6) at(2,4){$s^{\prime}$};
\node[circle,
minimum width =32pt ,draw=blue, fill=blue!15] (7) at(2,2){$h_{1}^{\prime}$};
\node[circle,
minimum width =32pt ,draw=blue, fill=blue!15] (8) at(2,0){$h_{2}^{\prime}$};
\node[circle,
minimum width =32pt ,draw=green, fill=green!15] (9) at(2,-2){$r^{\prime}$};
\draw[->] (0) --(2);
\draw[->] (0) --(3);
\draw[->] (0) --(4);
\draw[->] (1) --(6);
\draw[->] (1) --(7);
\draw[->] (2) --(6);
\draw[->] (2) --(7);
\draw[->] (3) --(7);
\draw[->] (3) --(8);
\draw[->] (4) --(8);
\draw[->] (5) to [out=300,in=60] (9);
\draw[->] (6) to [out=300,in=60] (9);
\draw[->] (8) -- (9);
\end{tikzpicture}
}
}
\caption{DAG representations for two sub-environments.}
\label{fig:G}
\end{figure}
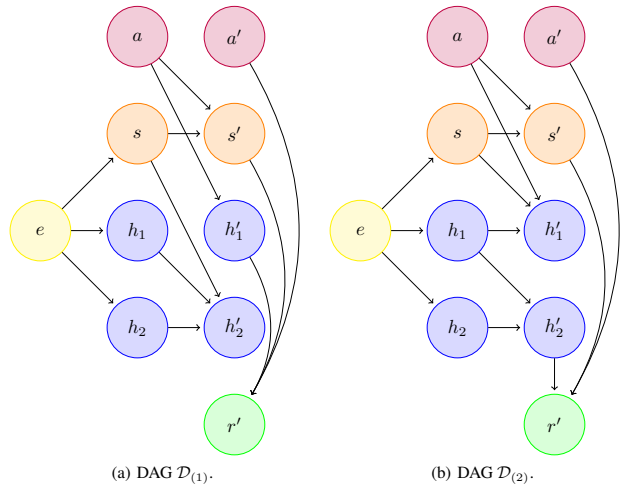

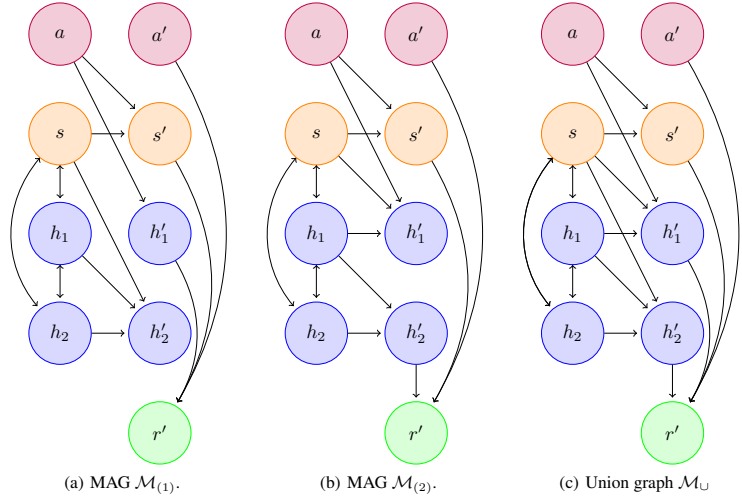
\begin{figure}[ht]
\centering
\resizebox{.6\linewidth}{!}{
\subfigure[MAG $\mathcal{M}_{(1)}$.]{
\begin{tikzpicture}[scale=0.9,
->,
shorten >=2pt]
\node[circle,
minimum width = 32pt, draw=purple, fill=purple!20] (1) at(0,6){$a$};
\node[circle,
minimum width = 32pt, draw=orange, fill=orange!20] (2) at(0,4){$s$};
\node[circle,
minimum width =32pt ,draw=blue, fill=blue!15] (3) at(0,2){$h_{1}$};
\node[circle,
minimum width =32pt ,draw=blue, fill=blue!15] (4) at(0,0){$h_{2}$};
\node[circle,
minimum width =32pt,draw=purple, fill=purple!20] (5) at(2,6){$a^{\prime}$};
\node[circle,
minimum width =32pt ,draw=orange, fill=orange!20] (6) at(2,4){$s^{\prime}$};
\node[circle,
minimum width =32pt ,draw=blue, fill=blue!15] (7) at(2,2){$h_{1}^{\prime}$};
\node[circle,
minimum width =32pt ,draw=blue, fill=blue!15] (8) at(2,0){$h_{2}^{\prime}$};
\node[circle,
minimum width =32pt ,draw=green, fill=green!15] (9) at(2,-2){$r^{\prime}$};
\draw[<->] (2) --(3);
\draw[<->] (3) --(4);
\draw[<->] (2) to [out=230,in=130] (4);
\draw[->] (1) --(6);
\draw[->] (1) --(7);
\draw[->] (2) --(6);
\draw[->] (2) --(8);
\draw[->] (3) --(8);
\draw[->] (4) --(8);
\draw[->] (5) to [out=300,in=60] (9);
\draw[->] (6) to [out=300,in=60] (9);
\draw[->] (7) to [out=300,in=60] (9);
\end{tikzpicture}
}
\subfigure[MAG $\mathcal{M}_{(2)}$.]{
\centering
\begin{tikzpicture}[scale=0.9,
->,
shorten >=2pt]
\node[circle,
minimum width = 32pt, draw=purple, fill=purple!20] (1) at(0,6){$a$};
\node[circle,
minimum width = 32pt, draw=orange, fill=orange!20] (2) at(0,4){$s$};
\node[circle,
minimum width =32pt ,draw=blue, fill=blue!15] (3) at(0,2){$h_{1}$};
\node[circle,
minimum width =32pt ,draw=blue, fill=blue!15] (4) at(0,0){$h_{2}$};
\node[circle,
minimum width =32pt,draw=purple, fill=purple!20] (5) at(2,6){$a^{\prime}$};
\node[circle,
minimum width =32pt ,draw=orange, fill=orange!20] (6) at(2,4){$s^{\prime}$};
\node[circle,
minimum width =32pt ,draw=blue, fill=blue!15] (7) at(2,2){$h_{1}^{\prime}$};
\node[circle,
minimum width =32pt ,draw=blue, fill=blue!15] (8) at(2,0){$h_{2}^{\prime}$};
\node[circle,
minimum width =32pt ,draw=green, fill=green!15] (9) at(2,-2){$r^{\prime}$};
\draw[<->] (2) --(3);
\draw[<->] (3) --(4);
\draw[<->] (2) to [out=230,in=130] (4);
\draw[->] (1) --(6);
\draw[->] (1) --(7);
\draw[->] (2) --(6);
\draw[->] (2) --(7);
\draw[->] (3) --(7);
\draw[->] (3) --(8);
\draw[->] (4) --(8);
\draw[->] (5) to [out=300,in=60] (9);
\draw[->] (6) to [out=300,in=60] (9);
\draw[->] (8) -- (9);
\end{tikzpicture}
}
\subfigure[Union graph $\mathcal{M}_{\cup}$]{
\begin{tikzpicture}[scale=0.9,
->,
shorten >=2pt]
\node[circle,
minimum width = 32pt, draw=purple, fill=purple!20] (1) at(0,6){$a$};
\node[circle,
minimum width = 32pt, draw=orange, fill=orange!20] (2) at(0,4){$s$};
\node[circle,
minimum width =32pt ,draw=blue, fill=blue!15] (3) at(0,2){$h_{1}$};
\node[circle,
minimum width =32pt ,draw=blue, fill=blue!15] (4) at(0,0){$h_{2}$};
\node[circle,
minimum width =32pt,draw=purple, fill=purple!20] (5) at(2,6){$a^{\prime}$};
\node[circle,
minimum width =32pt ,draw=orange, fill=orange!20] (6) at(2,4){$s^{\prime}$};
\node[circle,
minimum width =32pt ,draw=blue, fill=blue!15] (7) at(2,2){$h_{1}^{\prime}$};
\node[circle,
minimum width =32pt ,draw=blue, fill=blue!15] (8) at(2,0){$h_{2}^{\prime}$};
\node[circle,
minimum width =32pt ,draw=green, fill=green!15] (9) at(2,-2){$r^{\prime}$};
\draw[<->] (2) --(3);
\draw[<->] (3) --(4);
\draw[<->] (2) to [out=230,in=130] (4);
\draw[->] (1) --(6);
\draw[->] (1) --(7);
\draw[->] (2) --(6);
\draw[->] (2) --(7);
\draw[->] (3) --(7);
\draw[->] (3) --(8);
\draw[->] (4) --(8);
\draw[->] (5) to [out=300,in=60] (9);
\draw[->] (6) to [out=300,in=60] (9);
\draw[->] (8) -- (9);
\draw[<->] (2) to [out=230,in=130] (4);
\draw[->] (2) --(8);
\draw[->] (7) to [out=300,in=60] (9);
\end{tikzpicture}
}
}
\caption{MAG representations for two sub-environments and their union graph.}
\label{fig:M}
\end{figure}

In this paper, we characterize the nonstationarity as a mixture of stationary distributions. Formally, we take the following definition. 

\begin{definition}[Mixture of stationary distributions]
The marginal distribution of $\{{\bm{s}}, {\bm{h}}, {\bm{a}}, {\bm{s}}^{\prime}, {\bm{h}}^{\prime}, {\bm{a}}^{\prime}\}$ is a mixture of stationary distributions across environments, i.e.,
\[
P({\bm{s}}, {\bm{h}}, {\bm{a}}, {\bm{s}}^{\prime}, {\bm{h}}^{\prime}, {\bm{a}}^{\prime}) = \sum_{k} \pi_k P({\bm{s}}, {\bm{h}}, {\bm{a}}, {\bm{s}}^{\prime}, {\bm{h}}^{\prime}, {\bm{a}}^{\prime}\mid e=k),
\]
where $\pi_k$ denotes the probability that the sample is from the $k$-th environment varying over time, and $P({\bm{s}}, {\bm{h}}, {\bm{a}}, {\bm{s}}^{\prime}, {\bm{h}}^{\prime}, {\bm{a}}^{\prime}\mid e=k)$ is invariant over time.  
\end{definition}

\begin{proof}[Proof of Proposition \ref{thm:idn}]
The outline of the proof are as follows. We first construct a strict partial order $\pi$ on $V$. Then, we induce the MAGs $\mathcal{M}_{(1)}, \dots, \mathcal{M}_{(k)}$ from the DAGs $\mathcal{D}_{(1)}, \dots, \mathcal{D}_{(k)}$ by applying the rules defined in Algorithm \ref{algo:MAG}. We show the constructed partial order $\pi$ is \textit{compatible}, that $\forall k$, it holds that (a) $u\in \operatorname{an}(v) \Rightarrow u<_\pi v \text{ in } \mathcal{M}^{(k)}$; and (b) $ u \leftrightarrow v \Rightarrow u \not\lessgtr_\pi v  \text{ in } \mathcal{M}^{(k)}$. Finally we leverage the existing results  in \citep{saeed2020causal} to conclude that $\mathcal{M}_{\cup}$ is a MAG.

We define a relation $\pi$ on $V$ as following: for any variable $u\in\{{\bm{s}}, {\bm{h}}, {\bm{a}}\}$ and any variable $v\in\{{\bm{s}}^{\prime}, {\bm{h}}^{\prime}, {\bm{a}}^{\prime}\}$,  we have (i) $u <_{\pi} v$; (ii) $v<_{\pi} r^{\prime}$. To show the above defined $\pi$ is a strict partial order, we first notice that $\pi$ is irreflexive, because $u\not<_{\pi}u, v\not<_{\pi}v$ and $r^{\prime}\not<_{\pi}r^{\prime}$. The transitivity and asymmetry also hold by definition of $\pi$. Therefore, $\pi$ is a partial order on $V$.

The Algorithm \ref{algo:MAG} constructs an MAG from DAG with three steps. The first step is to add bidirected edges among the nodes in $\operatorname{ch}(e)$. Different values of $e$ leads different marginal distribution of ${\bm{s}}, {\bm{h}}$, hence $\operatorname{ch}(y)\subseteq \{{\bm{s}}, {\bm{h}}\}$. Therefore, the bidirected edges are added with both nodes belonging to $\{{\bm{s}}, {\bm{h}}\}$.  For the second step, there is no such node $t$, with $(t\rightarrow u)\in E$ and $(u\leftrightarrow v)\in B$, because the nodes in $\{{\bm{s}}, {\bm{h}}\}$ have no ancestor other than itself. So the second step adds the directed edges when $u=v$. The third step in our case is redundant.  Equation (\ref{eq:transition}) shows that there is no instantaneous causal effects in the system, so there is no $u,v$ such that $(u\leftrightarrow v)\in B$ while $u \in \mathrm{an}_{\mathcal{D}}(v)$. From all above, if the input of Algorithm \ref{algo:MAG} is $\mathcal{D}_{(k)}$, then it outputs a MAG $\mathcal{M}_{(k)} = (V, D_{(k)}, B_{(k)})$ with the set of nodes $V$, the set directed edges $D$ equals to the set of directed edges $E_{(k)}$ after removing the node $e$, and the set of bidirected edges $B_{(k)}$ consists edges among nodes in $\{{\bm{s}}, {\bm{h}}\}$. 

Then, we check the condition (a) and (b) to show $\mathcal{M}_{(1)}, \dots, \mathcal{M}_{(k)}$ are compatible with the above defined $\pi$. For (a), if $u$ is the ancestor node of $v$, then the structure of $\mathcal{M}_{(k)}$ implies that either $u\in \{{\bm{s}}, {\bm{h}}\}$ and $v\in \{{\bm{s}}^{\prime}, {\bm{h}}^{\prime}\}$ are nodes in $\{{\bm{s}}^{\prime}, {\bm{h}}^{\prime}\}$, or $u$ is a node from $\{{\bm{s}^{\prime}}, {\bm{h}^{\prime}}\}$ and $v=r^{\prime}$. For (b), if $u\leftrightarrow v$, then $u,v$ are nodes in $\{{\bm{s}}, {\bm{h}}\}$, hence $u \not\lessgtr_\pi v  \text{ in } \mathcal{M}^{(k)}$. These means that $\pi$ is a common strict partial order on $V$ for all MAGs. In this setup, we can leverage existing results from Lemma 4.3 \citep{saeed2020causal} to show that the environment-shared union graph $\mathcal{M}_{\cup}$ is also a maximal ancestral graph.

\end{proof}

\section{Toy Example under the Causal Interpretation}\label{sec:toy-example}

To better understand the causal interpretation for non-stationary RL, let's consider a simple toy example. Given a stationary environment with a state space represented as $(s_1, s_2)$. In this example, to maintain simplicity, we focus only on the state's mask, omitting the action mask and noise term. We define the original dynamics function as $f(\bm{c}_s\odot \bm{s}, a)=\bm{c}_s\odot \bm{s} + a$. For the toy environment, we consider a basic causal model wherein $s_i^\prime$ is only influenced by $s_i$. 

Consequently, the original mask is
\begin{equation}
\bm{c}_s=
\begin{pmatrix}
1 & 0\\
0 & 1
\end{pmatrix}
\end{equation}

Given this, we can derive that 
\begin{equation}
\begin{gathered}
    s_1^\prime=s_1+a\\
    s_2^\prime=s_2+a.
\end{gathered}
\end{equation}
We denote the non-stationarity in our experiments (Eq \ref{eq:non-stat}) simplistically as $\bm{s}^\prime=f(\bm{s},a)[1+n(t)]$, where $n(t)$ represents the introduced non-stationarity, which makes the dynamics becoming time-varying. In this scenario, the non-stationary environment's dynamics function becomes 
\begin{equation}
\begin{gathered}
s_1^\prime=s_1+a+(s_1+a)n(t)\\
s_2^\prime=s_2+a+(s_2+a)n(t).
\end{gathered}
\end{equation}
It is obvious that the dynamics introduces a time-varying term. In this context, we can define $h_i\doteq (s_i+a)n(t)$, leading to $s_i^\prime=s_i+h_i+a$. This allows us to deduce the dynamics of $\bm{h}$ as 
\begin{equation}
\begin{aligned}
h_i^\prime&=(s_i^\prime+a)\cdot n(t+1)\\
&=(s_i+h_i+2a)\cdot n(t+1)\\
&\doteq g_i(\bm{c}_s^t\odot \bm{s},\bm{c}_h^t\odot \bm{h},a),
\end{aligned}
\end{equation}
where $\bm{c}_s^t,\bm{c}_h^t$ symbolize the time-varying masks caused by non-stationarity $n(t)$. 

More specifically, we derive
\begin{equation}
\bm{c}_s^t=\bm{c}_h^t=
\begin{pmatrix}
n(t+1) & 0\\
0 & n(t+1)
\end{pmatrix}
\end{equation}

The masks represent the causal effects between variables after non-stationary changes occur in this example. The values represent the degree of causal effects, which can be treated as edge weights after normalization. Based on our theory, COREP's goal is to learn the union graph shared by these graph structures and edge weights during the change process. That is, through a common graph with edge weights, it includes all possible non-stationary changes. Therefore, the information of non-stationarity in the final learned graph is reflected both in the topology of the union graph and in the edge weights representing probabilities.

By introducing the time-varying masks and $\bm{h}$, we can make the dynamics function remains stationary, transferring non-stationarity to the causal model. Thus, we have provided a walk-through under the simple toy example. 

In fact, as illustrated in Figures \ref{fig:G} and \ref{fig:M}, there are more intricate causal relationships in complex environments. As depicted above, $\bm{h}$ can encapsulate not only the inherent environmental information but also the complex causal relationships with non-stationarity. Our proposed union MAG (Proposition \ref{thm:idn}) and the correspondingly designed dual-GAT architecture aim to learn such intricate causal models, enabling generic RL algorithms to handle non-stationarity under this causal representation.

\newpage

\section{Implementation and Training Details}\label{sec:implement-details}

\subsection{Pseudo code for COREP}

In Algorithm \ref{algo:COR}, we summarize the steps of COREP. For more specific details, please refer to the code provided in our supplementary material.

\begin{algorithm}[ht]
\caption{\textbf{C}ausal-\textbf{O}rigin \textbf{REP}resentation (\textbf{COREP})}\label{algo:COR}
\begin{algorithmic}[1]
\STATE \textbf{Init:} env; VAE parameters $\theta, \phi$; policy parameters: $\psi$; replay buffer $\mathcal{B}$; TD buffer $\mathcal{B}_\delta$.
\FOR{$i=0,1,\ldots $}
\STATE Collect trajectory $\tau_i$ with $\pi_\psi({\bm{a}}|{\bm{s}, \bm{h}})$.
\STATE Update replay buffer $\mathcal{B}[i] \gets \tau_i$.
\FOR{$j=0,1,\ldots $}
\STATE Sample a batch of episodes $E_j$ from $\mathcal{B}$ and TD errors $\{\delta_k\}$ from $\mathcal{B}_\delta$.
\STATE Transform states into $\bm{X}$ through MLPs.
\STATE Compute ${\bm{A_X}}=\mathrm{Softmax}\left({\bm{X}}{\bm{X}}^{\mathrm{T}}\odot({\bm{1}}_N-{\bm{I}}_N)\right)$.
\STATE Compute $\delta_\alpha=\left({\sum_{|{\mathcal{B}_\delta}| - \alpha|{\mathcal{B}_\delta}|<k<|{\mathcal{B}_\delta}|}\delta_k}\right) / {\alpha|{\mathcal{B}_\delta}|}$. 
\IF{$\delta_\alpha \notin (\mu_{\delta}-\eta\sigma_{\delta}, \mu_{\delta}+\eta\sigma_{\delta})$} 
\STATE unfreeze weights of core-GAT.
\ELSE
\STATE freeze weights of core-GAT.
\ENDIF
\STATE Get graph representation $\bm{G}_\text{core}, \bm{G}_\text{general}$ from core-GAT and general-GAT.
\STATE Compute $\mathcal{L}_\text{guide}, \mathcal{L}_\text{MAG}, \mathcal{L}_\text{sparsity}$ according to Equation (\ref{eq:l-guide}, \ref{eq:l-sparsity-mag}).
\STATE Input $\bm{G}(\bm{s}) = \bm{G}_\text{core}\oplus \bm{G}_\text{general}$ into VAE encoder $q_\phi$ and infer $\mu_{\bm{h}}, \sigma_{\bm{h}}$.
\STATE Sample ${\bm{h}}\sim \mathcal{N}\left(\mu_{\bm{h}}, \sigma_{\bm{h}}\right)$
\STATE Decode $\hat{\bm{s}}$ from ${\bm{h}}$ using decoder $p_\theta$, then compute $\mathcal{L}_\text{VAE}$ according to Equation (\ref{eq:l-vae}).
\STATE Do policy optimization for $\pi_\psi({\bm{a}}|{\bm{s}, \bm{h}})$, then compute $\mathcal{L}_\text{policy}$ and TD error ${\delta}$.
\STATE Compute $\mathcal{L}_{\text{total}}$ according to Equation (\ref{eq:objective}) and use it for gradient-updating $\theta,\phi,\psi$.
\STATE Push $\delta$ into TD buffer $\mathcal{B}_\delta$.
\ENDFOR
\ENDFOR
\end{algorithmic}
\end{algorithm}

\newpage

\subsection{Hyperparameters}

We list the hyperparameters for MLP, GAT, and VAE structures in Table \ref{tab:structure-parameters}, and the hyperparameters for policy optimization and training in Table \ref{tab:training-parameters}.

\begin{table}[ht]
\centering
\caption{Hyperparameters for the structure of MLP, GAT, and VAE.}
\resizebox{\textwidth}{!}{
\begin{tabular}{cc}
\toprule[1pt]
\textbf{Hyperparameter}            & \textbf{Value}  \\ \midrule[1pt]
MLP activation                 & ReLU            \\ 
MLP hidden dim                 & 512             \\ 
MLP learning rate              & 1e-3            \\ 
GAT activation                 & ELU             \\ 
GAT hidden dim                 & 32 (Cartpole Swingup, Reacher Easy/Hard, Cup Catch, Cheetah Run) \\
                               & 64 (Otherwise)  \\ 
GAT node numbers               & 4 (Cartpole Swingup, Reacher Easy/Hard, Cup Catch) \\
                               & 8 (Cheetah Run, Hopper Stand) \\
                               & 16 (Otherwise)  \\ 
node feature dim               & 16 (Cartpole Swingup, Reacher Easy/Hard, Cup Catch, Cheetah Run) \\
                               & 32 (Otherwise)  \\ 
GAT head numbers               & 2 (Quadruped Walk, Fish Upright, Walker Walk, Swimmer Swimmer6/15) \\
                               & 1 (Otherwise)   \\ 
VAE encoder hidden dim         & 128             \\ 
VAE decoder hidden dim         & 64              \\ 
latent representation dim      & 4 (Cartpole Swingup, Reacher Easy/Hard, Cup Catch) \\
                               & 8 (Cheetah Run, Hopper Stand) \\
                               & 16 (Otherwise)  \\ \bottomrule[1pt]
\end{tabular}
}
\label{tab:structure-parameters}
\end{table}

\begin{table}[ht]
\centering
\caption{Hyperparameters for policy optimization and training.}
\resizebox{\textwidth}{!}{
\begin{tabular}{cc}
\toprule[1pt]
\textbf{Hyperparameter}      & \textbf{Value}  \\ \midrule[1pt]
Policy hidden dim       & 256 (Swimmer Swimmer6/15, Walker Walk, Fish Upright, Quadruped Walk) \\
                        & 128 (Otherwise)  \\ 
Policy learning rate    & 7e-4            \\ 
$\lambda_\text{1}$ (for $\mathcal{L}_\text{guide}$)            & 0.1             \\ 
$\lambda_\text{2}$ (for $\mathcal{L}_\text{MAG}/\mathcal{L}_\text{sparsity}/\mathcal{L}_\text{VAE}$)         & 1e-3            \\ 
PPO update epoch        & 16              \\ 
PPO $\gamma$                   & 0.97            \\ 
PPO $\varepsilon$ clip                & 0.1             \\ 
TD buffer size          & 2000            \\ 
Confidence level $\eta$              & 1.96            \\ \bottomrule[1pt]
\end{tabular}
}
\label{tab:training-parameters}
\end{table}

\newpage

\section{Full Experiment Details}\label{sec:full-exp}

\subsection{Details about Environment Settings.}

\begin{figure}[ht]
    \centering
    \includegraphics[width=\linewidth]{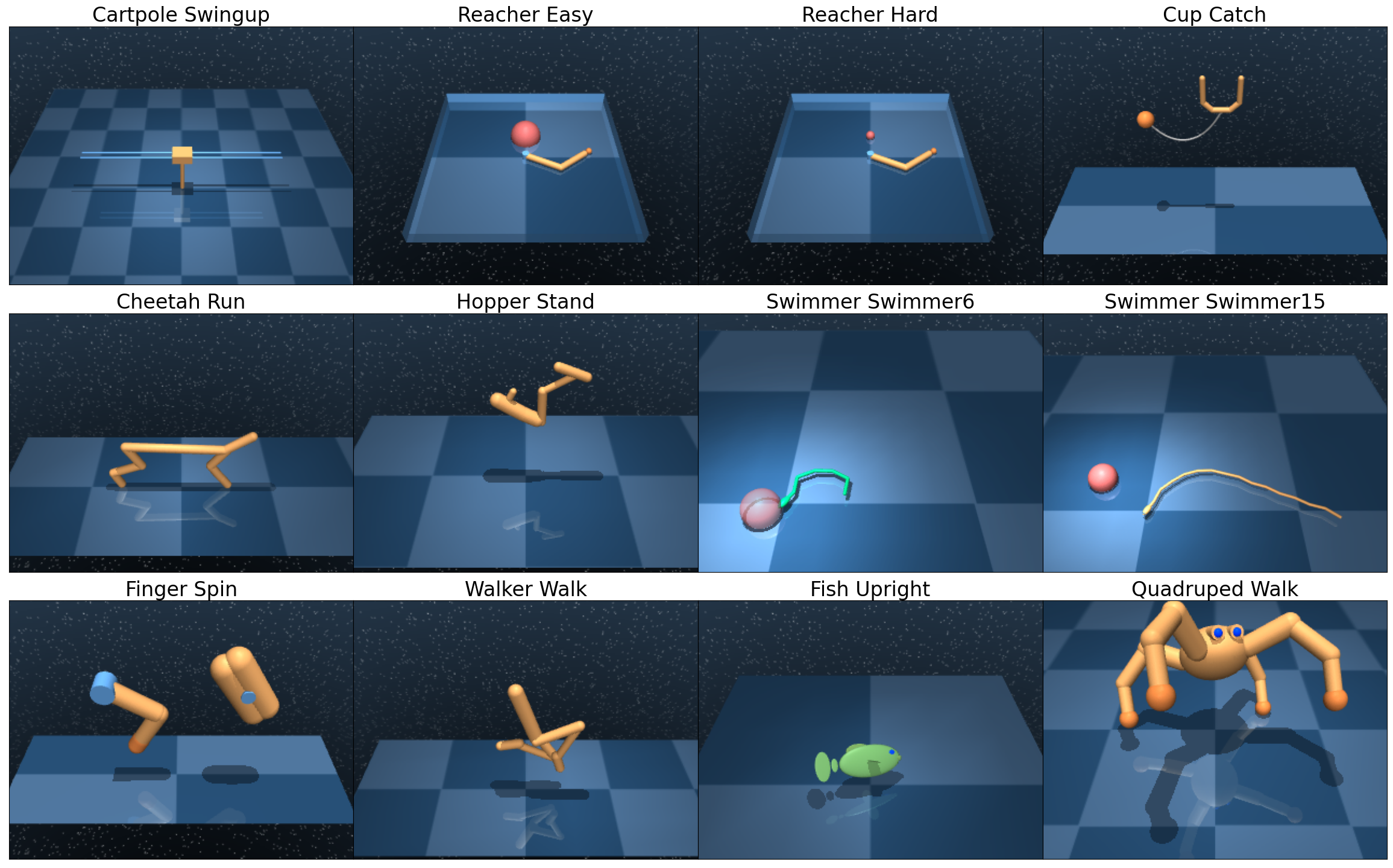}
    \caption{The environment we use in our experiment. We add non-stationary noise to the observations of these environments according to Equation (\ref{eq:non-stat}).}
    \label{fig:envs}
\end{figure}

Figure \ref{fig:envs} shows the environments we use in the experiment. We add non-stationary noise to the observations of these environments according to Equation (\ref{eq:non-stat}). These environments vary in terms of complexity, from low-dimensional problems like ``Reacher Easy'' to high-dimensional ones like ``Quadruped Walk''. All these tasks require the agent to understand and control its physical embodiment in order to achieve the desired goals. The specific descriptions of these environments and goals are as follows.

\textbf{Cartpole Swingup.} The cart can move along a one-dimensional track. The pole is attached to the cart with a joint allowing it to rotate freely. The initial state has the pole hanging down, and the goal is to apply forces to the cart such that the pole swings up and is balanced upright. Actions typically involve applying a horizontal force to the cart.

\textbf{Reacher Easy.} The agent is a two-joint robotic arm. The arm must move in a two-dimensional plane to touch a target position. The arm's state includes its joint angles and velocities. The action is the torque applied to each of the joints. The target's position is fixed in this version.

\textbf{Reacher Hard.} The task is the same as "Reacher Easy," but the target position is randomly placed in each episode, making the task more difficult as the agent has to learn to reach various positions.

\textbf{Cup Catch.} The agent is a robotic arm holding a cup, and there's a ball attached to the cup with a string. The arm needs to move in a way to swing the ball and catch it in the cup. The arm's state includes the position and velocity of the arm joints and the position and velocity of the ball. The actions are the torques applied at the arm's joints.

\textbf{Cheetah Run.} The agent is a model of a cheetah-like robot with 9 DoF(Degrees of Freedom): the agent can flex and extend its "spine," and each leg has two joints for flexing and extending. The agent's state includes the joint angles and velocities, and the actions are the torques applied to each of the joints. The goal is to move forward as fast as possible.

\textbf{Hopper Stand.} The agent is a one-legged robot, and its goal is to balance upright from a resting position. The agent's state includes the angle and angular velocity of the torso, as well as the joint angles and velocities. The actions are the torques applied to the joints.

\textbf{Swimmer Swimmer6.} The agent is a snake-like robot swimming in a two-dimensional plane. The robot has 6 joints, and the goal is to swim forward as fast as possible. The agent's state includes the joint angles and velocities, and the actions are the torques applied to the joints.

\textbf{Swimmer Swimmer15.} This is a more complex version of the Swimmer environment, with the agent being a 15-joint snake-like robot. Like the simpler version, the goal is to swim forward as fast as possible.

\textbf{Finger Spin.} The agent is a robot with two fingers, and there's a freely spinning object. The goal is to keep the object spinning and balanced on the fingertips. The state includes the positions and velocities of the fingers and the object, and the actions are the forces applied by the fingers.

\textbf{Walker Walk.} The agent is a bipedal robot, and the goal is to walk forward as fast as possible. The agent's state includes the angle and angular velocity of the torso, and the joint angles and velocities. The actions are the torques applied to the joints.

\textbf{Fish Upright.} The agent is a fish-like robot swimming in a three-dimensional fluid. The goal is to swim forward while maintaining an upright orientation. The agent's state includes the orientation and velocity of the fish, and the actions are the torques and forces applied to move the fish.

\textbf{Quadruped Walk.} The agent is a quadrupedal (four-legged) robot. Like the bipedal walker, the goal is to walk forward as fast as possible. The agent's state includes the angle and angular velocity of the torso, and the joint angles and velocities. The actions are the torques applied to the joints.

To ensure consistency in our conclusion, the experiments are conducted under various non-stationarity settings, which include `within-episode \& across-episode', `within-episode', and `across-episode' non-stationarities. These settings are respectively denoted as (W+A)-EP, W-EP, and A-EP. Specifically, these non-stationarities can be expressed as
\begin{equation}\label{eq:non-stat-full}
   s^\prime = f(s,a) + f(s,a)\cdot\alpha_d\left[c^t_1\cos(c_2^t\cdot t)+c_3^i\sin(c_4^i\cdot i)\right]
\end{equation}
\begin{equation}\label{eq:non-stat-w}
    s^\prime = f(s,a) + f(s,a)\cdot\alpha_d\left[c^t_1\cos(c_2^t\cdot t)\right]
\end{equation}
\begin{equation}\label{eq:non-stat-a}
    s^\prime = f(s,a) + f(s,a)\cdot\alpha_d\left[c_3^i\sin(c_4^i\cdot i)\right]
\end{equation}

We experimented only under (W+A)-EP when looking at the performance of the algorithm, while in a more detailed ablation study we experimented with all three different setings.

\subsection{Settings of Baselines}

For VariBAD, we meta-train the models (5000 batch size, 2 epochs for all experiments) and show the learning curves of meta-testing. The tasks parameters for meta-training are uniformly sampled from a Gaussian distribution $\mathcal{N}(0,1)$.

For all approaches, we use the same backbone algorithm for policy optimization, \ie, PPO with the same hyperparameters, as shown in Table \ref{tab:training-parameters}.

\newpage

\subsection{Full Results of Performance}

\begin{figure}[ht]
    \centering
    \includegraphics[width=\linewidth]{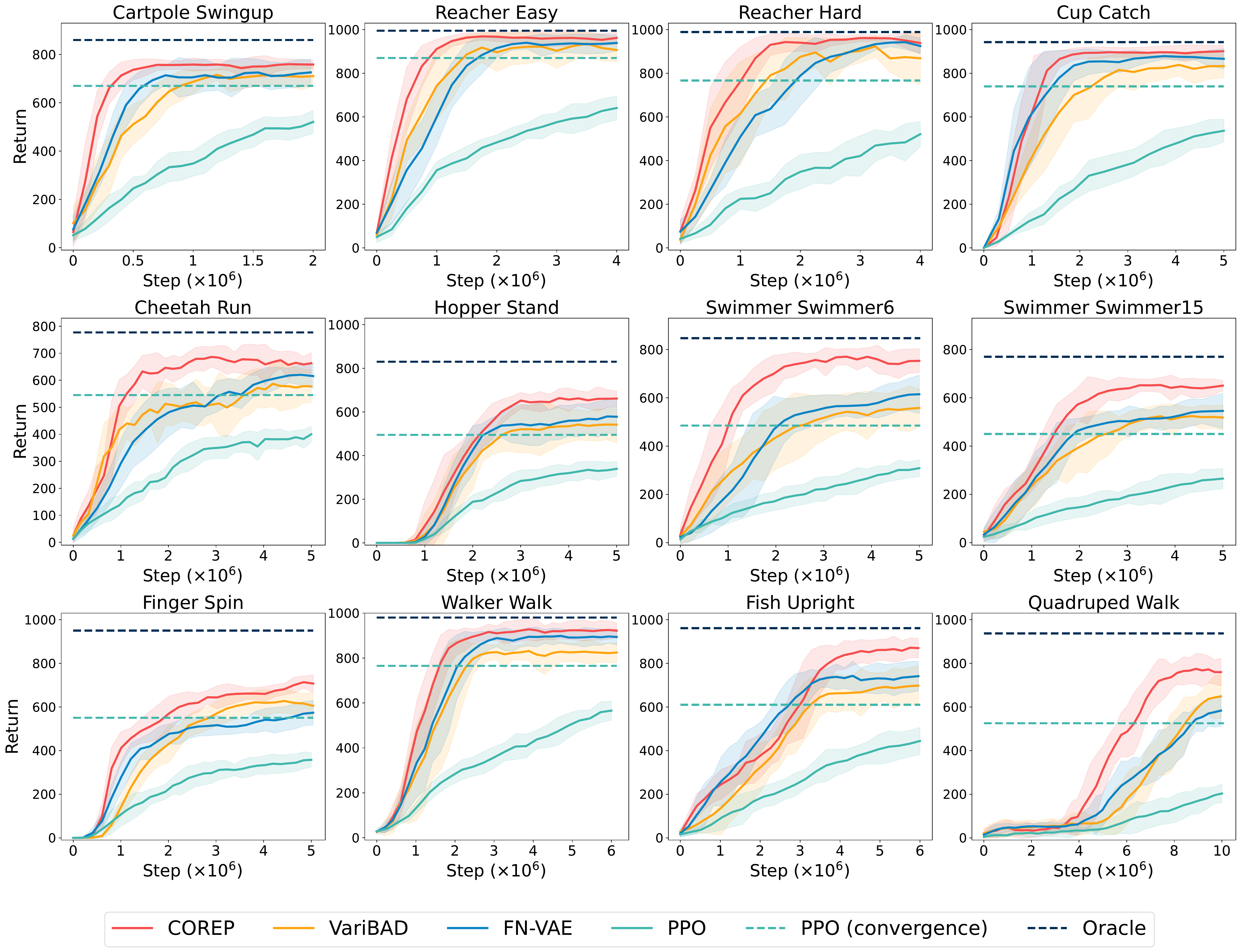}
    \caption{Learning curves of COREP and baselines in different environments. Solid curves indicate the mean of all trials with 5 different seeds. Shaded regions correspond to standard deviation among trials. The dashed lines represent the asymptotic performance of PPO and Oracle.}
    \label{fig:multi-env-perf-full}
\end{figure}

Figure \ref{fig:multi-env-perf-full} shows the full learning curves. We add non-stationary noise as Equation (\ref{eq:non-stat-full}) to all environments. According to the results, COREP consistently performs well in environments of different complexities, proving the effectiveness of the algorithm. Especially in \textit{Hopper Stand, Swimmer Swimmer6, Swimmer Swimmer15, Finger Spin, Fish Upright, and Quadruped Walk}, COREP demonstrates a larger performance gap, highlighting its superiority over baselines.

FN-VAE has the ability to approach our COREP in some simple environments (\textit{Cartpole Swingup, Reacher Easy, Reacher Hard, and Cup Catch}), but still exhibits significant variance, reflecting its instability, especially in more complex environments where it performs even worse than VariBAD (\textit{Finger Spin, Quadruped Walk}). VariBAD shows a large performance gap and variance in all environments, indicating poor stability to non-stationarity. PPO's performance is consistently the worst across all environments due to the lack of any optimization for non-stationarity.

\newpage

\subsection{Full Results of Ablation Study}

\begin{figure}[ht]
    \centering
    \includegraphics[width=\linewidth]{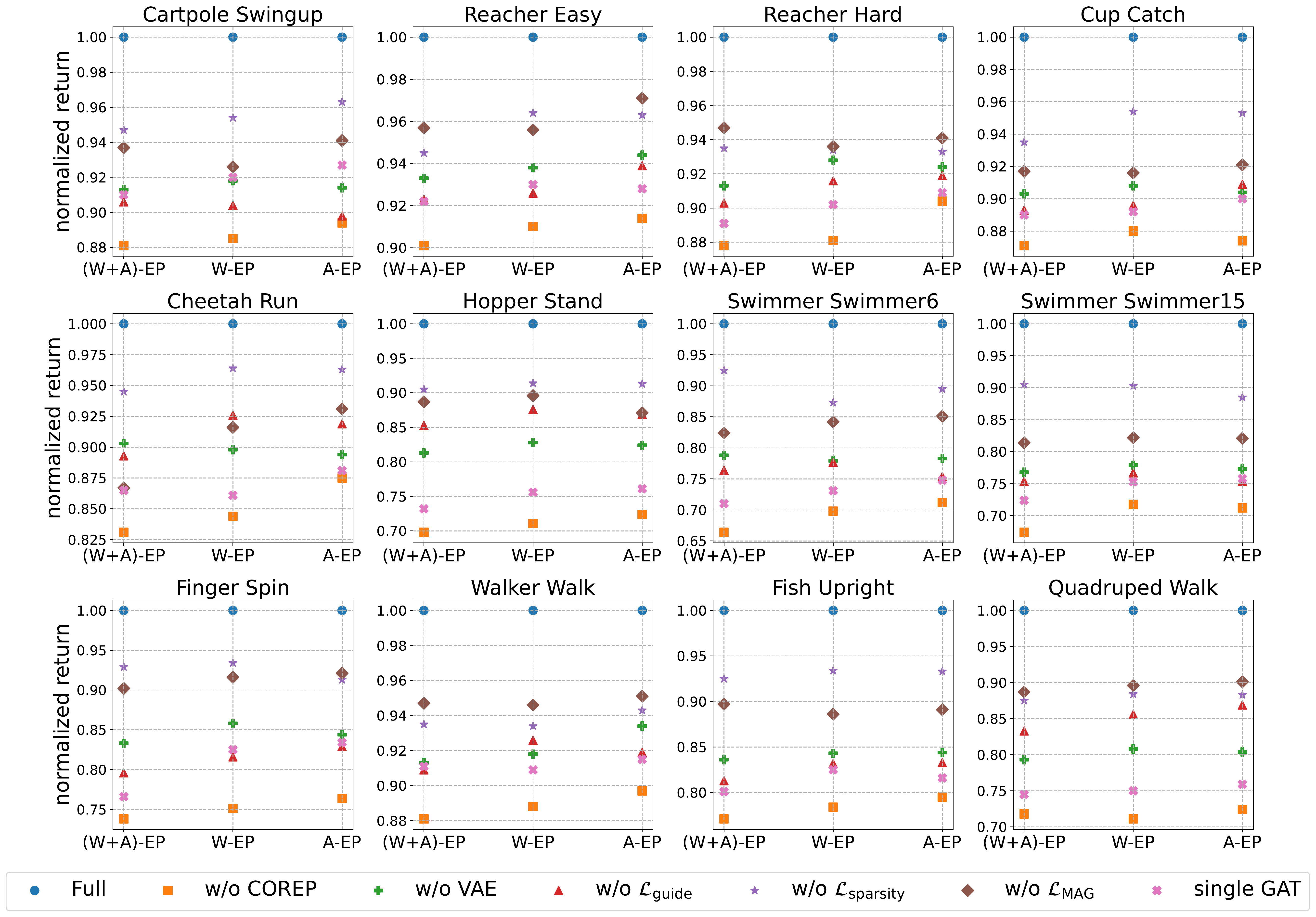}
    \caption{Final mean returns of 3 different trials on all environments with different components and non-stationarity settings. Returns are normalized to the full version of COREP in each environment.}
    \label{fig:ablation-full}
\end{figure}

Figure \ref{fig:ablation-full} shows the performance after removing different components in COREP. All (W+A)-EP, W-EP, and A-EP non-stationary noises, \ie Equation (\ref{eq:non-stat-full}, \ref{eq:non-stat-w}, \ref{eq:non-stat-a}), are separately added to the environments. Each point of Figure \ref{fig:ablation-full} represents the normalized return, which is used to observe the contribution of removed components to the overall algorithm. Specifically,
\begin{itemize}
    \item `w/o COREP' remove all COREP-specific designs and retaining only the VAE process;
    \item `w/o VAE' is a version without the VAE process;
    \item `w/o $\mathcal{L}_\text{guide}$' removes the guided update mechanism containing TD detection;
    \item `w/o $\mathcal{L}_\text{sparsity}$' removes the corresponding loss $\mathcal{L}_\text{sparsity}$ in Equation (\ref{eq:objective});
    \item `w/o $\mathcal{L}_\text{MAG}$' removes the corresponding loss $\mathcal{L}_\text{MAG}$ in Equation (\ref{eq:objective});
    \item `single GAT' maintains the same network structure without introducing a secondary GAT and corresponding update mechanism.
\end{itemize}

\newpage

\subsection{Full Results on Non-stationarity Degrees}

\begin{figure}[ht]
    \centering
    \includegraphics[width=\linewidth]{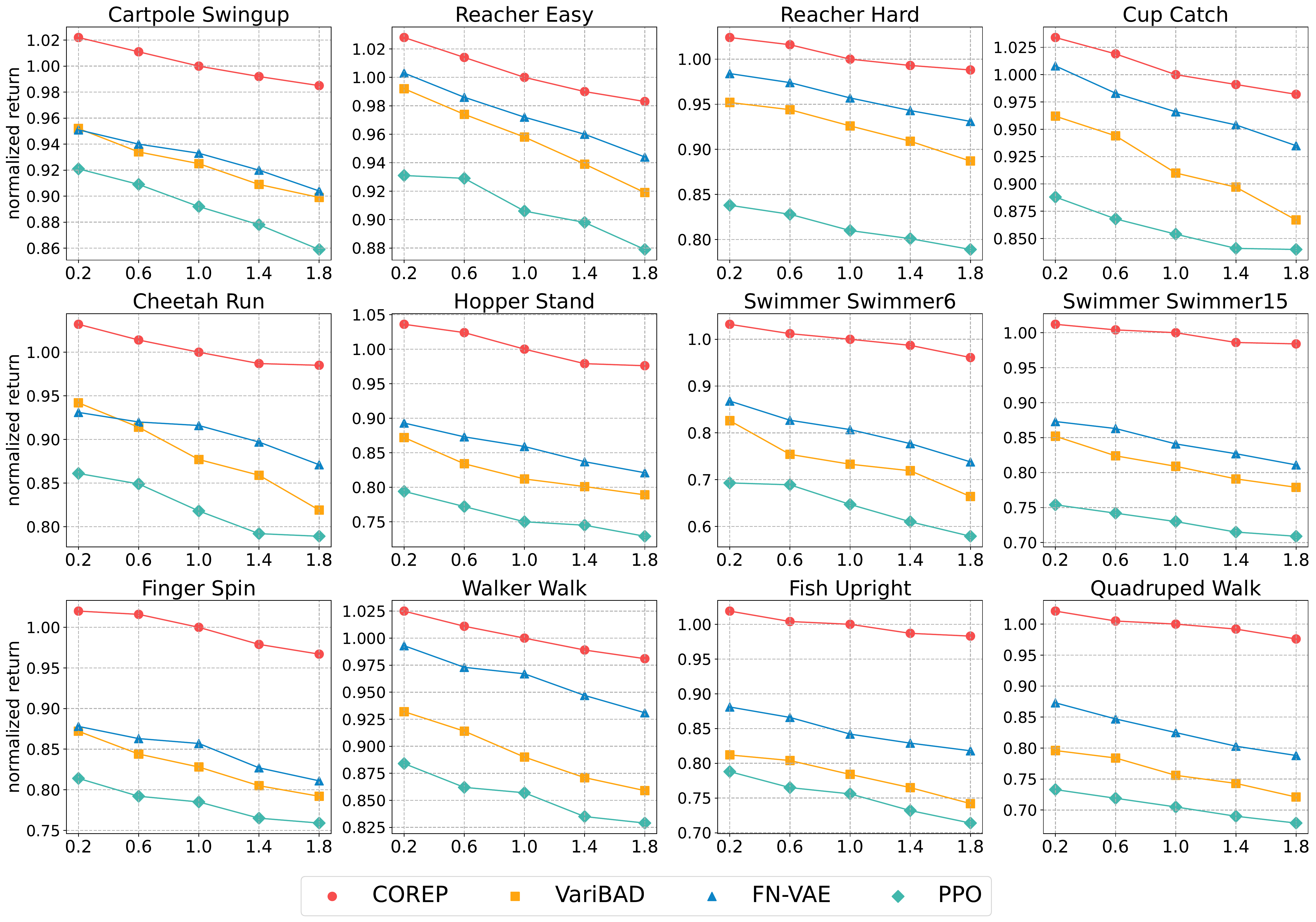}
    \caption{Final mean returns of 3 different trials on Cheetah Run and Swimmer Swimmer6 environments with different degrees of non-stationarity. Returns are normalized to the COREP algorithm with standard degree $1.0$.}
    \label{fig:non-stat-full}
\end{figure}

To analyze the impact of varying degrees of non-stationarity in the environment, we change the values of $\alpha_d$ in Equation (\ref{eq:non-stat-full}). As depicted in Figure \ref{fig:non-stat-full}, the results suggest that the performance of the compared baselines is more affected by the degree of non-stationarity. Conversely, COREP exhibits consistent performance when encountering different degrees of non-stationarity, further affirming our claim that COREP can effectively tackle more complex non-stationarity.

\newpage

\subsection{Visualization of Learned Graph}\label{sec:vis}

To visualize the graph learned by the COREP algorithm, we respectively show the weighted adjacency matrices of core-GAT and general-GAT after 5M steps in the \textit{Cartpole Swingup}, \textit{Reacher Hard}, and \textit{Cup Catch} environment in Figure \ref{fig:vis-cartpole-swingup}, \ref{fig:vis-reacher-hard}, \ref{fig:vis-cup-catch}. 
It is noteworthy that the number of graph nodes is set to be the same as the observation dimension of each environment to bring better empirical insights into how the dual graph actually functions. For further information about the actual meaning of each dimension in different environments, please refer to the DeepMind Control Suite technical report \citep{tassa2018deepmind}.

In these heatmaps, each value on a grid represents the weight of an edge from a node on the y-axis to a node on the x-axis. A higher value indicates a greater causal influence.

Based on the results, it can be seen that core-GAT indeed focuses more on a few core nodes in its learned graph structure, while general-GAT compensates for some overlooked detailed information by core-GAT. The results align well with our claim made in the manuscript.

\begin{figure}[ht]
    \centering
    \includegraphics[width=.75\linewidth]{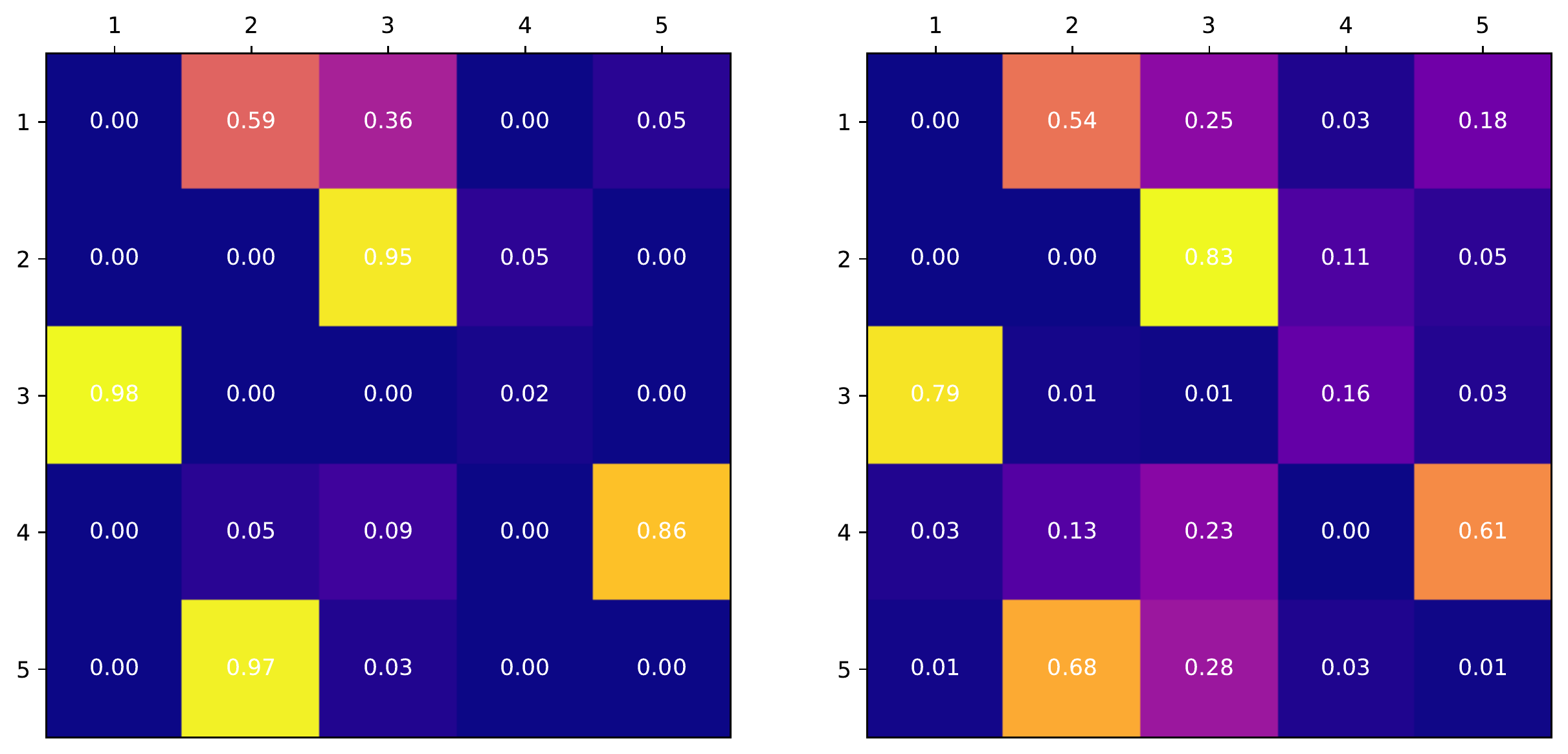}
    \caption{Weighted adjacency matrix of core-GAT (\textit{left}) and general-GAT (\textit{right}) in Cartpole Swingup after 5M steps.}
    \label{fig:vis-cartpole-swingup}
\end{figure}

\begin{figure}[htbp]
    \centering
    \includegraphics[width=.75\linewidth]{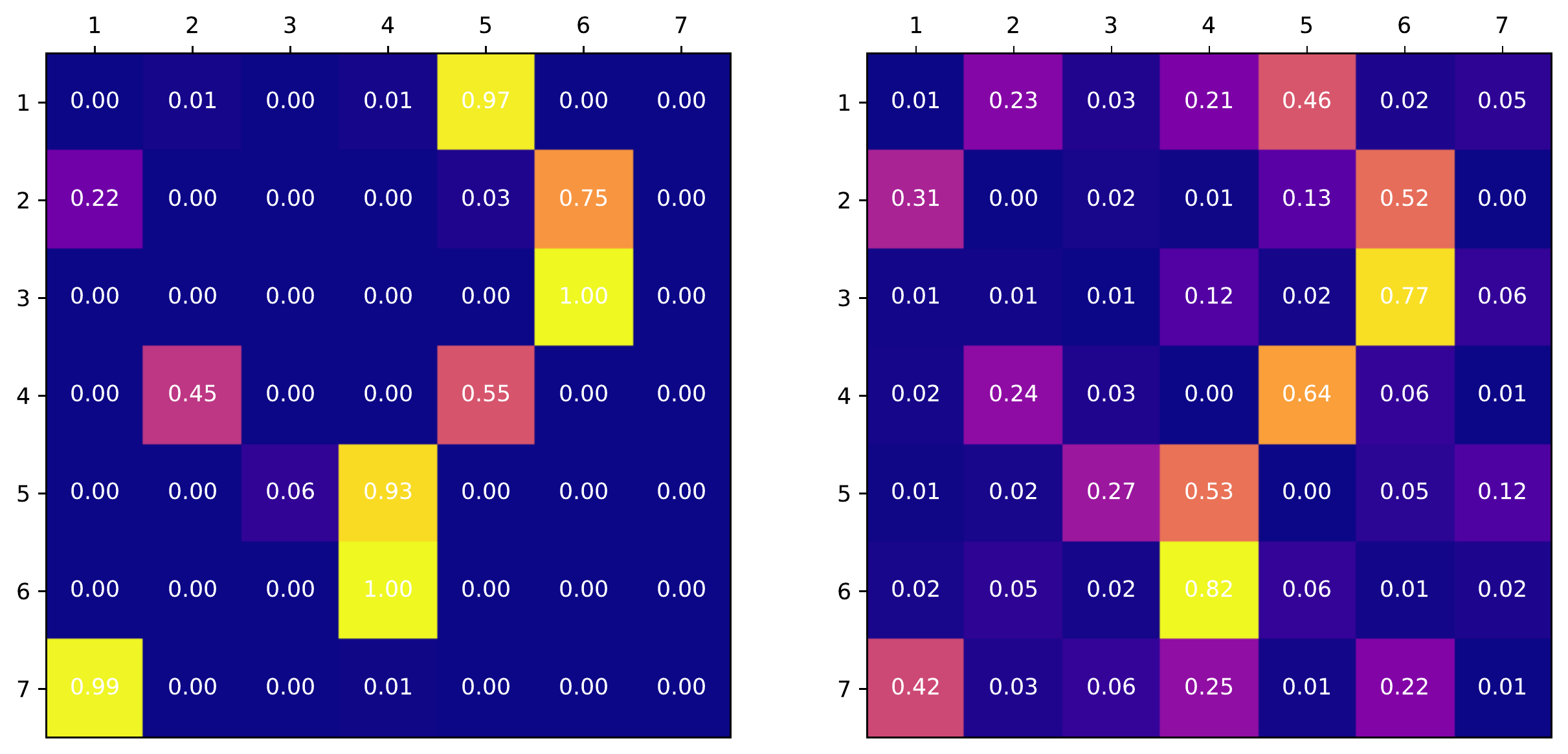}
    \caption{Weighted adjacency matrix of core-GAT (\textit{left}) and general-GAT (\textit{right}) in Reacher Hard after 5M steps.}
    \label{fig:vis-reacher-hard}
\end{figure}

\begin{figure}[htbp]
    \centering
    \includegraphics[width=.75\linewidth]{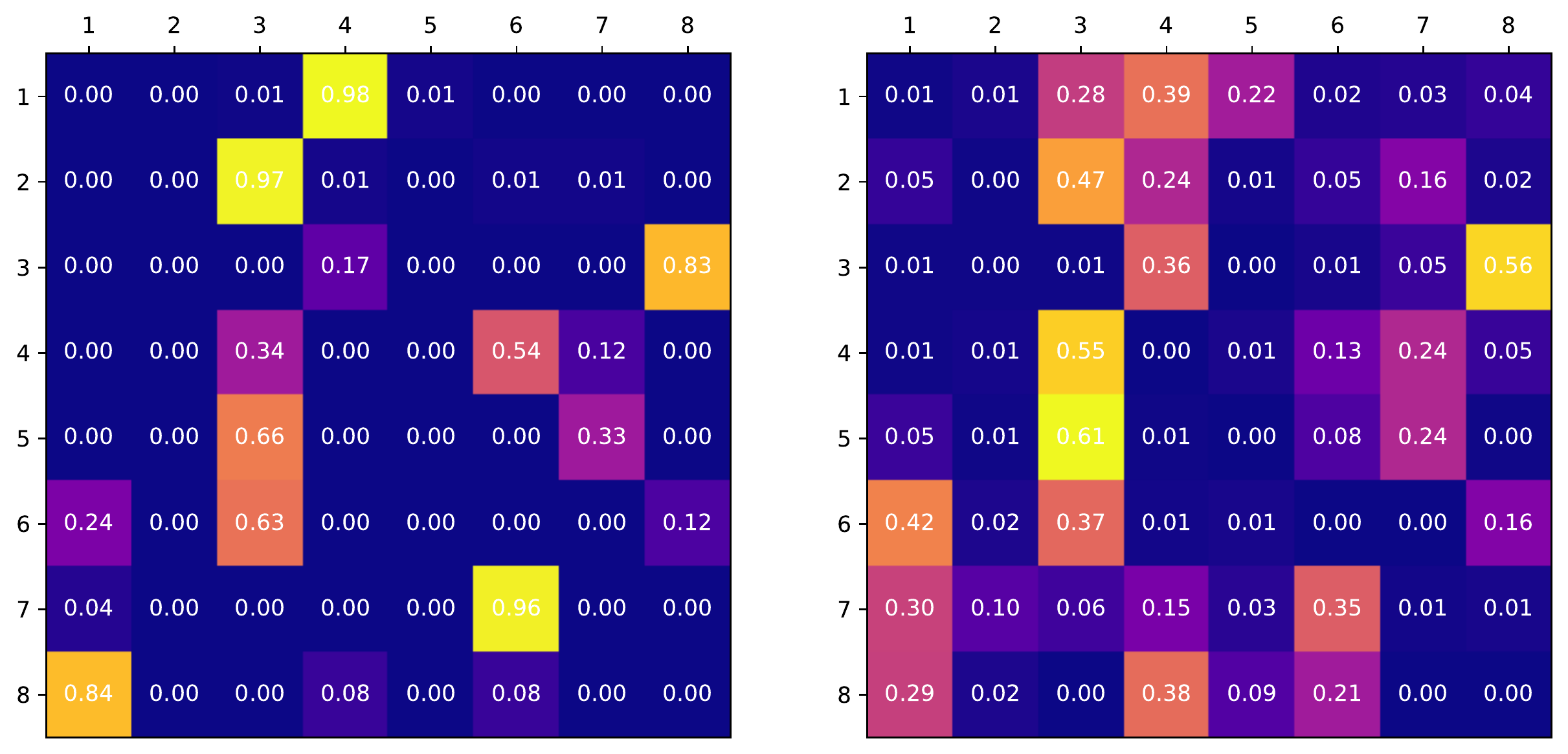}
    \caption{Weighted adjacency matrix of core-GAT (\textit{left}) and general-GAT (\textit{right}) in Cup Catch after 5M steps.}
    \label{fig:vis-cup-catch}
\end{figure}

\newpage

\subsection{Study on Tunning Parameters of Loss Terms}\label{sec:tuning}

We conduct additional experiments to study COREP's sensitivity to the hyperparameters $\lambda_1, \lambda_2$ in the objective function (\ref{eq:objective}). Results are shown in Table \ref{tab:sen-l1} and Table \ref{tab:sen-l2}. 

The results indicate that the performance is not sensitive to the values of these hyperparameters. Even with extensive adjustments to these parameters, COREP consistently surpasses the performance of SOTA baseline. Notably, we observe instances in certain environments where adjustments lead to even higher performance. This suggests that with more precise tuning, there is potential to further enhance COREP's effectiveness.

These findings are significant as they demonstrate that COREP's superior performance is largely attributed to its designs in structure and process, rather than relying on hyperparameter tuning. The low sensitivity to hyperparameter values also implies ease of use and adaptability in diverse settings.

\begin{table}[htbp]
\centering
\caption{Sensitivity results of parameter $\lambda_1$. The results indicate that the performance is not sensitive to $\lambda_1$, even with significant adjustments to its value, COREP still outperforms FN-VAE.}\label{tab:sen-l1}
\resizebox{\textwidth}{!}{
\begin{tabular}{ccccccc}
\toprule[1pt]
$\lambda_1$         & Cartpole Swingup   & Reacher Easy       & Reacher Hard       & Cup   Catch        & Cheetah Run        & Hopper   Stand     \\
\midrule[1pt]
$0.5$            & $ 732.3 \pm 30.6 $ & $ 947.6 \pm 24.6 $ & $ 938.4 \pm 29.2 $ & $ 868.5 \pm 24.7 $ & $ 634.8 \pm 41.6 $ & $ 651.4 \pm 29.3 $ \\
$0.1$ (original) & $ \bm{743.4 \pm 21.2} $ & $ \bm{964.6 \pm 17.3} $ & $ 947.2 \pm 23.1 $ & $ \bm{877.5 \pm 19.2} $ & $ \bm{651.1 \pm 44.3} $ & $ 645.5 \pm 25.8 $ \\
$0.05$           & $ 732.8 \pm 26.6 $ & $ 939.3 \pm 16.4 $ & $ \bm{949.4 \pm 25.7} $ & $ 870.1 \pm 19.8 $ & $ 642.7 \pm 52.8 $ & $ \bm{656.7 \pm 32.3} $ \\
$0.01$           & $ 722.3 \pm 28.4 $ & $ 945.2 \pm 23.3 $ & $ 934.3 \pm 19.5 $ & $ 864.6 \pm 25.5 $ & $ 629.4 \pm 57.2 $ & $ 641.3 \pm 33.8 $ \\
$0.001$          & $ 717.5 \pm 31.9 $ & $ 928.8 \pm 25.8 $ & $ 936.9 \pm 25.7 $ & $ 858.3 \pm 18.6 $ & $ 622.8 \pm 49.2 $ & $ 627.5 \pm 27.5 $ \\
\midrule
FN-VAE         & $ 710.3 \pm 64.5 $ & $ 913.3 \pm 38.7 $ & $ 928.1 \pm 21.9 $ & $ 851.3 \pm 31.6 $ & $ 606.5 \pm 75.3 $ & $ 580.9 \pm 47.3 $ \\
\bottomrule[1pt]
\end{tabular}
}
\end{table}

\begin{table}[htbp]
\centering
\caption{Sensitivity analysis of parameters $\lambda_2$. The results indicate that the performance is not sensitive to $\lambda_2$, even with significant adjustments to its value, COREP still outperforms FN-VAE.}\label{tab:sen-l2}
\resizebox{\textwidth}{!}{
\begin{tabular}{ccccccc}
\toprule[1pt]
$\lambda_2$ & Cartpole Swingup   & Reacher Easy       & Reacher Hard       & Cup Catch          & Cheetah Run        & Hopper Stand       \\
\midrule[1pt]
$0.01$              & $ 728.7 \pm 36.1 $ & $ 936.3 \pm 23.7 $ & $ 931.5 \pm 24.4 $ & $ 864.9 \pm 25.8 $ & $ 632.2 \pm 37.5 $ & $ 623.5 \pm 32.1 $ \\
$0.005$             & $ \bm{745.6 \pm 34.8} $ & $ 955.3 \pm 25.3 $ & $ \bm{953.6 \pm 22.1} $ & $ 872.5 \pm 21.5 $ & $ 644.5 \pm 36.1 $ & $ 638.1 \pm 24.6 $ \\
$0.001$ (original)  & $ 743.4 \pm 21.2 $ & $ \bm{964.6 \pm 17.3} $ & $ 947.2 \pm 23.1 $ & $ \bm{877.5 \pm 19.2} $ & $ \bm{651.1 \pm 44.3} $ & $ \bm{645.5 \pm 25.8} $ \\
$0.0005$            & $ 718.6 \pm 31.5 $ & $ 946.6 \pm 19.6 $ & $ 945.5 \pm 19.4 $ & $ 856.3 \pm 18.7 $ & $ 621.6 \pm 48.4 $ & $ 641.1 \pm 29.6 $ \\
$0.0001$            & $ 721.1 \pm 35.9 $ & $ 949.5 \pm 27.8 $ & $ 940.6 \pm 19.7 $ & $ 859.2 \pm 26.4 $ & $ 618.1 \pm 55.7 $ & $ 619.5 \pm 31.7 $ \\
\midrule
FN-VAE              & $ 710.3 \pm 64.5 $ & $ 913.3 \pm 38.7 $ & $ 928.1 \pm 21.9 $ & $ 851.3 \pm 31.6 $ & $ 606.5 \pm 75.3 $ & $ 580.9 \pm 47.3 $ \\
\bottomrule[1pt]
\end{tabular}
}
\end{table}

\newpage

\section{Compute Resource Details}

We list the hardware resources used in Table \ref{tab:computing-resources}, and list the training time required for a single trial in each environment in Table \ref{tab:run-time}.

\begin{table}[ht]
\centering
\caption{Computational resources for our experiments.}
\begin{tabular}{ccc}
\toprule[1pt]
\textbf{CPU }                     & \textbf{GPU}                      & \textbf{RAM}   \\
\midrule[1pt]
Intel I9-12900K@3.2GHz (24 Cores) & Nvidia RTX 3090 (24GB) $\times$ 2 & 256GB\\
\bottomrule[1pt]
\end{tabular}
\label{tab:computing-resources}
\end{table}

\begin{table}[ht]
\centering
\caption{Computing time of each single trial in different environments.}
\begin{tabular}{cc}
\toprule[1pt]
\textbf{Environment}       & \textbf{Training Time}  \\ \midrule[1pt]
Cartpole Swingup           & 12 hours                \\
Reacher Easy               & 16 hours                \\
Reacher Hard               & 20 hours                \\
Cup Catch                  & 16 hours                \\
Cheetah Run                & 24 hours                \\
Hopper Stand               & 28 hours                \\
Swimmer Swimmer6           & 28 hours                \\
Swimmer Swimmer15          & 36 hours                \\
Finger Spin                & 28 hours                \\
Walker Walk                & 28 hours                \\
Fish Upright               & 30 hours                \\
Quadruped Walk             & 42 hours                \\ \bottomrule[1pt]
\end{tabular}
\label{tab:run-time}
\end{table}

\section{Licenses}

In our code, we have used the following libraries which are covered by the corresponding licenses:

\begin{itemize}
    \item Numpy (BSD-3-Clause license)
    \item PyTorch (BSD-3-Clause license)
    \item PyTorch Geometric (MIT license)
    \item DeepMind Control (Apache-2.0 license)
    \item OpenAI Gym (MIT License)
\end{itemize}

\end{document}